\documentclass{article} 
\usepackage{iclr2024_conference,times}

\usepackage{hyperref}

\usepackage[utf8]{inputenc} 
\usepackage[T1]{fontenc}    
\hypersetup{
    colorlinks,
    linkcolor={red!50!black},
    citecolor={blue!50!black},
    urlcolor={blue!80!black}
}
\usepackage{url}            
\usepackage{booktabs}       
\usepackage{amsfonts}       
\usepackage{nicefrac}       
\usepackage{microtype}      
\usepackage{xcolor}         
\usepackage{cjhebrew}  
\usepackage{natbib}  
\newlength\mylen
\setcitestyle{numbers,square,citesep={,\kern-\mylen}}
\usepackage{dirtytalk}  
\usepackage{graphicx}
\usepackage{subcaption}  
\usepackage{amsmath}
\usepackage{multirow}
\usepackage{enumitem}
\usepackage{amssymb,amsmath,amsthm,dsfont,bm}
\usepackage{xurl}
\usepackage{comment}
\usepackage{mathtools}
\usepackage[bottom]{footmisc}
\usepackage{amsthm}
\newtheorem{theorem}{Theorem}
\usepackage[T1]{fontenc}
\usepackage[scaled=0.85]{sourcecodepro}

\renewcommand{\paragraph}[1]{\textbf{#1}\,\,}

\newcommand{\red}[1]{{\leavevmode\color{red}{#1}}}
\newcommand{\blue}[1]{{\leavevmode\color{blue}{#1}}}

\newcommand{\green}[1]{{\leavevmode\color[RGB]{0,128,0}{#1}}}

\newcount\Comments
\usepackage{color}
\definecolor{darkgreen}{rgb}{0,0.5,0}
\newcommand{\kibitz}[2]{\ifnum\Comments=1{\color{#1}{#2}}\fi}
\newcommand{\gn}[1]{\kibitz{orange}{[Gali: #1]}}
\newcommand{\dcp}[1]{\kibitz{purple}{[DCP: #1]}}

\definecolor{antiquefuchsia}{rgb}{0.57, 0.36, 0.51}
\newcount\CommentsA
\newcommand{\kibitzA}[2]{\ifnum\CommentsA=1{\color{#1}{#2}}\fi}

\newcommand{\dcpadd}[1]{#1}
\newcommand{\niradd}[1]{#1}
\newcommand{\galiadd}[1]{#1}

\newcommand\todo[1]{\red{TODO: {#1}}}
\newcommand\tocite{\red{[CITE]}}

\newcommand{\extended}[1]{} 

\newcommand\expect[2]{\mathbb{E}_{#1}{\left[ {#2} \right]}}
\newcommand\prob[2]{P_{#1}{\left( {#2} \right)}}
\DeclareMathOperator*{\argmax}{argmax}
\DeclareMathOperator*{\softmax}{softmax}
\DeclareMathOperator*{\argmin}{argmin}

\newcommand{\one}[1]{\mathds{1}{\{{#1}\}}}

\newcommand{\naive}{na\"ive}

\newcommand{\R}{\mathbb{R}}
\newcommand{\yhat}{{\hat{y}}}

\newcommand{\ybar}{{\bar{y}}}

\newcommand{\pval}{{\tilde{v}}}

\newcommand{\abar}{{\bar{a}}}

\newcommand{\thetahat}{{\hat{\theta}}}
\newcommand{\fhat}{{\hat{f}}}
\newcommand{\ptilde}{{\tilde{p}}}
\newcommand{\nhat}{{\hat{n}}}
\newcommand{\pihat}{{\hat{\pi}}}

\newcommand{\mask}{{\mu}}

\newcommand{\pref}{{\beta}}

\newcommand{\alloc}{{\mathtt{alloc}}}
\newcommand{\choice}{{\mathtt{choice}}}
\newcommand{\selection}{{\mathtt{selection}}}

\newcommand{\decongestion}{{\mathtt{decongestion}}}
\newcommand{\welf}{{W}}
\newcommand{\welfprxy}{{\widetilde{W}}}

\newcommand{\Loss}{{\mathcal{L}}}
\newcommand{\dist}{{\mathcal{D}}}
\newcommand{\smplst}{{\mathcal{S}}}
\newcommand{\masktemp}{{\tau}}
\newcommand{\vmin}{{v_{\textnormal{min}}}}
\newcommand{\vmax}{{v_{\textnormal{max}}}}

\newcommand{\method}[1]{{\fontfamily{lmtt}\selectfont{{#1}}}}
\newcommand{\DbR}{\method{DbR}}
\newcommand{\DbRpolicy}{\DbR$(\pihat)$}
\newcommand{\DbRmask}{\DbR$({\hat{\mask}})$}
\newcommand{\DbRtopk}{\DbR$(\thetahat)$} 
\newcommand{\pricepred}{\method{price-pred}}
\newcommand{\choicepred}{\method{choice-pred}}
\newcommand{\random}{\method{random}}
\newcommand{\oracle}{\method{oracle}}

\newtheorem{proposition}{Proposition}
\newtheorem{definition}{Definition}

\newtheorem{lemma}{Lemma}

\title{Decongestion by Representation: Learning\\
to Improve Economic Welfare in  Marketplaces}

\author{%
  Omer Nahum\\
  Technion DDS \\
  \And
  Gali Noti \\
  Cornell CS\\
  \And
  David C. Parkes \\
  Harvard SEAS \\
  \And
  Nir Rosenfeld \\
  Technion CS \\
}

\iclrfinalcopy 
\begin{document}

\maketitle

\begin{abstract}
 
Congestion is a common failure mode of markets,  where consumers compete inefficiently on the same subset of goods
(e.g.,  chasing the same small set of properties on a vacation rental platform).
The typical economic story is that prices  decongest by 
balancing supply and demand. 
But in modern online marketplaces, prices are typically set in a decentralized way by sellers, \galiadd{and the information about items is inevitably partial.} 
The power of a platform  is limited to controlling \emph{representations}---the 
\niradd{subset of information about items presented by default to users.}
%
This motivates the present study of \emph{decongestion by representation},
where a platform seeks to learn representations that reduce  congestion and thus improve social welfare.
The technical challenge is twofold: relying only on revealed preferences from the choices of consumers, rather than true preferences; and the combinatorial problem associated with
 representations that  determine \dcpadd{the  features to reveal in the default view}.
 We tackle both challenges by proposing a 
{\em differentiable proxy of welfare} that can be trained end-to-end on consumer choice data.
We develop sufficient conditions for when decongestion promotes welfare,
and present \dcpadd{the results of extensive} experiments on both synthetic and real data
that demonstrate the utility of our approach.
\looseness=-1


\end{abstract}

\section{Introduction}


Online marketplaces have become ubiquitous as our primary means
for consuming tangible goods as well as services across many different domains.
Examples span a variety of commercial segments, including
dining (e.g., Yelp),
real estate (e.g., Zillow),
vacation homestays (e.g., Airbnb),
used or vintage items (e.g., eBay),
handmade crafts (e.g., Etsy),
and specialized freelance labor (e.g., Upwork).
A key reason underlying the success of these platforms is their ability to manage an exceptionally large and diverse collections of items,
to which users are given immediate and seamless access.
Once a desired item has been found on a platform then
obtaining it should---in principle---be only `one click away.'\looseness=-1

But just like conventional markets, online markets are also prone to certain forms of market failure,
which 
may hinder the ability of users to  easily obtain valued items.
One prominent type of failure, which our paper targets, is \emph{congestion}.
Congestion occurs when demand for certain items exceeds supply;
i.e., when
multiple users are interested in a single item
of which there are not sufficiently-many copies available.
E.g., 
in vacation rentals, the same 
 vacation home may draw the interest of many users, \galiadd{but only one of them can rent it.} 
This can prevent potential transactions from materializing, resulting
 in reduced social welfare---to the detriment of 
 \galiadd{users, suppliers, and the platform itself}.


In conventional markets,
the usual economic response to congestion 
is to set prices
in an appropriate manner (e.g., \cite{shapley1971assignment}). 
In our example, if the attractive vacation home is priced correctly,
then only one user (who values it most, relative to other properties) 
will choose it;
similarly, if other items are also priced correctly in relation to user valuations,
then prices can fully decongest the market and 
the market can obtain optimal welfare, \galiadd{defined as the sum of users' valuations to their assigned items.} 
\looseness=-1

But for modern online markets, this approach is \niradd{unattainable} for two reasons.
The first  reason is that many 
online platforms  do not have control over prices, which are
instead set in a decentralized way by
different sellers.
%
The second reason is more subtle, but central to the solution we advance in this paper: 
we argue that an inherent aspect of online markets
is that users make choices under limited information,
and that this limits the effectiveness of price.
Online environments impose natural constraints on the amount of information consumed,
 due to technical limitations
(e.g., restricted screen space),
behavioral mechanisms
(e.g., cognitive capacity, attention span, impatience),
or design choices
(e.g., what information is highlighted, appears first \dcpadd{as a default}, or requires less effort to access).
As such, the decisions of users are more affected by whatever information is made more readily available to them.
\dcpadd{From an economic perspective, this means they are making decisions under `incorrect' preference models,
for which (i) prices that decongest (or clear) at these erroneous preferences are incorrect, and
(ii) prices that decongest at correct preferences still leave congestion at erroneous preferences}.\footnote{Economic theory has many examples of other ways in which partial information hurts markets (e.g.,~\citep{akerlof1978market}).}%
\looseness=-1

Partial information is therefore a reality that platforms must cope with---a new reality which requires new approaches.
To ease congestion and improve welfare,
our main thesis is that platforms can---and \emph{should}---utilize their control over \emph{information},
and in particular, on how items are represented to users.
The decision of \emph{representation}---the default way in which items are shown to users---is typically in the hands of the platform;
and while providing equal access to all information may be the ideal,
reality dictates that choosing \emph{some} representation is inevitable.\footnote{In the influential book `Nudge', Thaler and Sunstein (\citeyear{nudge}) argue similarly for `choice architecture' at large.\looseness=-1}
Given this,
\niradd{we propose to use machine learning to solve the necessary design problem of choosing beneficial item representations}.

To this end, we present a new framework for learning item representations that  reduce congestion and promote welfare.
\dcpadd{Since congestion results from users making choices independently
according to their own individual preferences,
to decongest, the platform must act to (indirectly) coordinate these idiosyncratic choices;}
and since representations affect choices by shaping how users `perceive' 
value,
we will seek to coordinate perceived preferences.
The basic premise of our approach is that,
with enough variation in true user preferences, 
it should be possible to find representations for which choices made
under perceived values remain both valuable \emph{and} diverse.
For example, consider a rental unit represented as having `sea view' and `sunny balcony' and draws the attention of many users
but does not convey other information such as `noisy location';
if users vary enough in how  they value quietness,
then showing `quiet' instead of `balcony' may help reduce congestion and improve outcomes.\looseness=-1


%

%
\extended{
A \naive\ approach would be to choose representations on the basis of feature importance, e.g., by training a predictive model of user choices and applying feature selection.
But the caveat in this approach is that it accounts only for demand---not supply, and to properly decongest a market,
a designated approach is needed.
}

From a learning perspective, the fundamental challenge 
is that welfare itself (and its underlying choices)
depends on private user preferences. 
For this, we
develop a proxy objective that relies on \galiadd{observable} choice data alone,
and optimizes for representations that encourage 
favorable decongested solutions through users' choices.
A technical challenge is that representations are combinatorial objects, corresponding to a  subset of features to show. 
Building on recent advances in differentiable discrete optimization, 
we modify our objective to be differentiable,
thus permitting end-to-end training using gradient methods. 
To provide formal grounding for our approach of decongestion by representation,
we theoretically study the connection between decongestion and welfare.
Using competitive equilibrium analysis,
we give several simple and interpretable sufficient conditions under which reducing congestion provably improves welfare.
Intuitively, this happens when it is possible to present item features across which user preferences are more diverse, while at the same time hiding features that are not too meaningful for the users. 
The conditions 
provide basic insight as to when our approach works well.\looseness=-1
%
%

We 
end with an extensive set of experiments that shed light on our proposed setting and learning approach. We first make use of  synthetic data to
explore the usefulness of decongestion as a proxy for welfare,
considering the importance of  preference variation,
the role of prices, and the degree of information partiality.
We then use real data of user ratings to elicit user preferences across a set of diverse items. Coupling this with simulated user behavior, 
we  
demonstrate the susceptibility of \naive\ prediction-based methods to harmful congestion,
and the ability of our congestion-aware representation learning framework to improve economic outcomes.
Code for all experiments can be found at:
\url{https://github.com/omer6nahum/Decongestion-by-Representation}.

\extended{\todo{plug in public github url}}


\if
\green{\rule{\linewidth}{0.5em}}


\dcp{I suggest to make the intro more punchy and direct, and to follow the sequence of args I drafted in the intro: (1) congestion a common problem, what it is,  and why it's a problem. (2) why prices may be insufficient. bring actual marketplace examples here. (3) irrational attention to some subset of features. representaiton inevitable. (4) learn to decongest. (5) technical challenges, including revealed preferences, fixed goods and IID users, and counterfactuals, and results.}

\todo{clarify that it's inevitable to choose some mask!}

Online marketplaces have become ubiquitous as our primary means
for consuming tangible goods and services
across a multitude of domains spanning many commercial segments.
To provide useful service to their users,
modern platforms have come to rely on machine learning to make predictions about user preferences.
Indeed, machine learning has proven to be quite effective in providing precise and timely predictions.
But this success has also made it tempting to forget that predictions are only a means---not an end, 
and that predictions alone, even if accurate, may not suffice for driving 
important objectives that are central to conventional markets, such as user welfare. \gn{Social welfare? (actually, we do social welfare, and user welfare is social welfare minus total revenue. But we could also write more freely here...}

To see why, consider for example a platform that suggests restaurants.
Given sufficient data, the platform may learn to predict with high precision
which restaurants users will like, and provide these as input to users.
As individuals, this should aid users in making better decisions.
But in terms of collective outcomes,
there is an inherent drawback to this approach:
Since users often have similar tastes,
popular restaurants are likely to draw more interested diners than they are able to accommodate;
this can then lead to unhappy customers inside the restaurant
(since it is crowded and noisy),
and unhappier customers outside it (since lines are long and frustrating, and not everyone will get a seat).
Thus, conventional learning may help users with what they \emph{want}---but not what they \emph{get}. \gn{I like the bottom line, but the part about reduction in value to those who did get what they wanted (although interesting) is not like this in our model/story. Also, we could maybe argue the opposite: that those who got a rare seat are even happier with it (VIP).}

The above exemplifies a common form of market failure known as \emph{congestion},
wherein demand---in the form of users individually making choices, exceeds supply---which is limited. \gn{Not sure, but it feels to me like ``demand exceeds supply'' takes to different places, and we would like to put the emphasis more on the part that people choose similar things and will lose from it (also individually).}
Other online marketplaces where congestion is likely to arise include
vacation homestays (e.g., Airbnb),
real estate (e.g., Zillow),
handmade crafts (e.g., Etsy),
used or vintage items (e.g., Ebay),
and specialized freelance labor (e.g., Upwork).
In all of these, supply is fundamentally limited,
and congestion is a material risk that can substantially impede user welfare.
Our focus herein is on such markets.

\todo{talk about supply-side congestion? seller welfare?} \gn{Yes, I think that in the above paragraph, if we say ``user welfare'' (rather than social welfare) it would make sense to at least mention something about the seller welfare and maybe also the platform's. E.g., add ``... impede user and seller welfare as well as lead to direct and indirect losses to the platform itself.'' Also, maybe change: ``congestion is a material risk'' to ``congested user choices are a material risk''.}

A natural way to increase welfare in congested markets is via \emph{decongestion};
in our example, this can be achieved for instance by encouraging 
some users to instead go to their second-favorite restaurant
(assuming this entails more variation in choices).
Classic market theory proposes 
two means by which a system can act to decongest a market:
it can either determine allocations (i.e., decide who gets what),
or it can set item prices (typically increasing with demand).
If the system has access to users' (private) valuation functions,
and can control either allocations or prices,
then market theory shows how to 
efficiently compute welfare-optimizing allocations and/or prices,
which guarantee decongestion, as well as market clearing \tocite.
\gn{This depends on the parameters. So maybe rephrase to: ``...which guarantees minimum congestion, as well as market clearing (assuming sufficient demand)''}
But in current online marketplaces, prices are typically set by suppliers,
users are free to make their own choices,\footnote{In our example, clearly the system cannot simply assign users to restaurants, nor tamper with their menus.}
and true user preferences are generally unknown.
We should therefore ask: what \emph{can} the system do to reduce congestion? \gn{Emphasis on the ``system'' too? Also, perhaps ``marketplace platform'' instead of a system? And if we put it this way -- it seems even more important to mention/discuss the platform's interest in reducing congestion (i.e., that congestion leads to losses on the side of the platform too.}

\todo{how to dodge recommendations? careful!} \gn{Right... David?...}

Our main thesis in this paper is that systems can utilize their control over \emph{information} to reduce congestion and improve welfare---and our goal is to show how machine learning can be leveraged to achieve this.
Classic theory assumes that users make rational choices on the basis of full information regarding items.
But in modern contexts, users almost never observe items in their entirety:
restricted screen space, bounded cognitive capacity, limited attention spans and impatience---all these place natural constraints on the amount of information regarding items that can be conveyed to users.
As such, users make decisions based on how they `perceive' value,
as it is shaped by partial information about items.
Clearly, \emph{what} information is revealed can be highly influential of user choices---and therefore, should directly affect congestion.
Our key observation is that it is typically in the system's power to determine \emph{how} items are represented to users,
and through this---to reduce congestion.
\gn{It is important to put more emphasis on the fact that presenting part of the information is inevitable. Systems always need to present only part of the information, and our observation/thesis is that it can do it in a way that increases welfare. It's important because otherwise it may sound like we suggest the system to ``lie'' for better welfare.}

\todo{congestion is product of irrational behavior - which is why we can't expect prices to decongest "on their own"}

In this paper we initiate the study of decongestion via information selection,
and propose a learning framework for optimizing user welfare by learning decongesting item representations.
We focus on representations that are \emph{truthful} but \emph{lossy},
in that they must reveal an item's true attributes, but cannot reveal all of them \tocite.
In this setting, and for a certain model of user choice,
we design a learning objective that enables to approximately optimize welfare through decongestion.
Thus, rather than learning which items to show (which we take as given),
we focus on learning how to show them.

The challenge in optimizing welfare is that welfare itself, as well as the underlying choices of users, depends on user valuations---which are private.
Working in the challenging but more realistic setting of \emph{revealed preferences}
in which data includes only past user choices (rather than full preference profiles),
we propose a proxy objective that sidesteps the reliance on valuations
by instead optimizing for representations that encourage 
favorable decongested solutions.
The technical challenge in optimizing this proxy is that representations in our setting are combinatorial objects;
towards this, and building on recent advances in differentiable discrete optimization \tocite,
we modify our objective to be differentiable,
thus permitting end-to-end training using gradient methods.\looseness=-1

In general, decongestion should work well as a proxy for welfare when there is sufficient variation in user preferences.
When this holds, representations can help disperse user choices;
even if users end up choosing (slightly) less beneficial items,
the fact that overall more items are allocated contributes to improved welfare. \gn{Consider rephrasing: currently it sounds like we ``sacrifice'' users for the overall welfare, but this is not exactly the case here.}
To formalize this idea, we propose and study notions of \emph{monotonicity},
which characterize when and how reducing congestion increases welfare, and 
give sufficient conditions for a market to be monotone.

We end with a series of experiments that shed light on our proposed setting and learning approach.
Using synthetic data, we explore the usefulness of decongestion as a proxy for welfare across multiple settings.
Here we consider the importance of natural preference variation,
the role of prices, and \todo{...}.
Then, using real data coupled with simulated user behavior,
we empirically evaluate the capacity of our approach to improve welfare by learning useful representations.
Our results demonstrate the susceptibility of \naive\ prediction-based methods to harmful congestion,
and indicate the potential power of congestion-aware representation learning to remedy this.

\todo{
\textbf{Broader perspective/aims}. \\
- state that \textbf{representation is *inevitable*}; must choose \emph{some} $k$ features. ground argument. \\
- connect to idea of choice architecture in nudge theory/book
(=can't escape determining the context for people making choices; exmple - checkbox is either opt-in or opt-out, must decide which) \\ 
- our idea works due to capacity to harness natural variation in user prefs \\
- congestion due to irrational - limited info \\
- rational prices, irrational behavior - users don't pay attention to "right" features
}

\fi
\subsection{Related work}

There is a growing recognition that many online platforms 
provide economic marketplaces,
and considerable efforts have been dedicated to studying the role of recommender systems in this regard~\citep{chaney2018algorithmic,schmit2018human,tabibian2019consequential}.
Some work, for example, has  studied the effects of  learning and recommendation on
the equilibrium dynamics of two-sided markets of users and suppliers
\citep{ben2018game,mladenov2020optimizing,jagadeesan2022supply},
exchange markets \citep{guo2022learning},
or markets of competing platforms \citep{ben2019regression,jagadeesan2023competition}.
The main distinction is that our paper studies not what to show to users, but how.
One study examined
the effect of  the complete absence of knowledge about some items on welfare~\citep{dean2020recommendations};
in contrast, we study how welfare is affected by partial information about items. 
There are also studies on the role of information in the form of
 recommendations
in enhancing system-wide performance with
learning users; 
e.g., the use of selective information 
reporting to promote exploration by a group of users~\cite{kremer2014implementing,mansour2015bayesian,bahar2020fiduciary}.
Again, this is quite distinct from our setting.\looseness=-1 

Conceptually related to our work is research in the field 
of  human-centered AI \cite{riedl2019human} that studies
 AI-aided human decision making, and in particular
 prior work that has  considered
 methods to learn 
 representations of inputs to 
 decision problems to aid in single-user decision making \cite{hilgard2021learning}.
 \extended{\dcp{say something about technical connection, if any?}}
Related, there is  work on
selectively providing information in the form of advice to a user in order to optimize their decision performance~\cite{noti2022learning}. 
It has also been argued 
that   providing less accurate predictive information to  users can sometimes improve performance \cite{bansal2021most}. These works, however, do not consider interactions between multiple users which are at the center of the types of markets we consider here.  

Though underexplored in  online markets,
several works in related fields have considered how representations
affect decisions. For example, \citep{kleinberg2019simplicity}
aim to establish the role of `simplicity' in decision-making aids,
and in relation to fairness and equity.
Works in strategic learning have emphasized the role of users
in representations; i.e., in learning to choose in face of strategic representations \citep{nair2022strategic}, 
and as controlling representations themselves \citep{krishnaswamy2021classification}.
Here we extend the discussion on representations to markets.\looseness=-1

\section{Problem Setup} \label{sec:setup}

The main element of our setting is a \emph{market},  where each market is composed of $m$ indivisible items and $n$ users. Within a market, items $j$ are described by non-negative feature vectors $x_j \in \R^d_+$ and prices $p_j \ge 0$.
Let  $X \in \R^{m \times d}$ denote all item features,
and $p \in \R^m$  denote all prices,
\niradd{which we assume to be fixed}.\footnote{%
\niradd{For example, this is reasonable when sellers adapt slowly
(or not at all; e.g., as in ad auctions \citep{doraszelski2018just,alcobendas2021adjustment}),
or when prices are set for a broader aggregate market.
For discussion on adaptive prices, see Appendix~\ref{apx:adaptive_prices}.}}
We mostly consider unit supply, in which there is only one unit of each item (e.g., as in vacation homes), but note our method directly extends to general finite supply, which we discuss later.\looseness=-1

Each user $i$ in a market has a {\em valuation function},
$v_i(x)$, which determines their true value for an item with
feature vector $x$.
We use $v_{ij}$ to denote user $i$'s value for the $j$th item.
We model each user with a non-negative, linear preference, 
with $v_i(x)=\pref_i^\top x$ for some
{\em user type}, $\pref_i \ge 0$. The effect is that 
 $v_i(x) \ge 0$ for all items,
and all item attributes contribute positively to value.
We assume w.l.o.g.~that values are scaled such that $v_i(x) \le 1$.
Users have unit demand, i.e.,
are interested in purchasing a single item.
Given full information on items,
a rational agent would choose $y_i^* = \argmax\nolimits_j v_{ij} - p_j$. 



\paragraph{Partial information.}
The unique aspect of our setup is that users make choices on the basis of partial information, over 
 which the system has control. For this, we model users as making decisions on an item
with feature vector $x$ based on its \emph{representation} $z$, which is
truthful but lossy:  $z$ must contain only information from $x$, but 
does not contain all of the information. 
We consider this to be a necessary aspect of a practical market, where users
are unable to appreciate all of the complexity of goods in the market.
Concretely, $z$ reveals a subset of $k \le d$ features from $x$,
where the set of features is
determined by a binary \emph{feature mask}, $\mask \in \{0,1\}^d$, with $|\mask|_1=k$,
that is fixed for all items.
Each  mask induces
\emph{perceived values}, $\pval$,
which are the values a user infers from observable features:
\begin{equation}
\label{eq:percieved_values_and_choices}
\pval_i(x) = \pref_i^\top (x \odot \mask) = (\pref_i)_\mask^\top z, 
\end{equation}
where $\odot$ denotes element-wise product,
and $(\pref)_\mask$ is $\beta$ restricted to features in $\mask$.
For market items $x_j$ we use $\pval_{ij}=\pref_i^\top (x_j \odot \mask)$.
Given this, under \emph{partial} information,
user $i$ makes choices $y_i$ via:\footnote{In Appx.~\ref{apx:synth-additional}
we experiment with an alternative decision model,
which shows qualitatively similar results.\looseness=-1}
\begin{align}
\label{eq:perceived_choice}
y_i(\mask) =
\choice(X,p;v_i\mask) 
\coloneqq \argmax\nolimits_j \pval_{ij} - p_j,
\end{align}
where $\pval_{ij} - p_j$ is agent $i$'s perceived utility from item $j$, with $y_i(\mask)=0$ encoding the `no choice' option, which occurs when no item has positive perceived utility.
\niradd{Note Eq.~\eqref{eq:perceived_choice} is a (simple) special case of \citep{wu2019feature}.}
When clear from context, we will drop the notational dependence on $\mask$.
We use $y \in \{0,1\}^{n \times m}$ to describe all choices in the market, where $y_{ij}=1$ if user $i$ chose item $j$, and 0 otherwise.\looseness=-1

Under Eq.~\eqref{eq:perceived_choice},
each user is modeled as a 
conservative boundedly-rational decision-maker,
whose perception of value derives from how items are represented, and in particular, by which features are shown.
Note that together with our positivity assumptions,
this ensures that representations cannot be used to portray items as more valuable than they truly are---which could lead to choices with negative utility.
\looseness=-1
%

\paragraph{Allocation.}
To model the effect of congestion we 
require an  {\em allocation mechanism}, denoted $a=\alloc(y_1,\dots,y_n)$, 
where
$a \in \{0,1\}^{n \times m}$ has 
$a_{ij}=1$ if item $j$ is allocated to user $i$, and 0 otherwise.
We will use $a(\mask)$ to denote allocations that result from choices $y(\mask)$.
We require {\em feasible allocations}, such that each item is allocated at most once and
each user receives at most one item.
For the allocation mechanism,  we use the 
\emph{random single round} rule,
where each 
item $j$  in demand 
 is allocated uniformly at random to one
of the users for which $y_i=j$.
This models 
congestion: if several users choose the same
 item $j$,
then only one of them receives it
 while all others receive nothing.
%
%
Intuitively, for welfare to be high, we would like that:
(i) allocated items give high value to their users, and
(ii) many items are allocated. As we will see, this observation forms the basis of our approach.
\looseness=-1

 \paragraph{Learning representations.}
To facilitate learning, we will assume there is some (unknown) distribution over markets, $\dist$,
from which we observe samples.
In particular, we model markets with a fixed set of items,
and users sampled iid from some pool of users.
For motivation, consider vacation rentals,
where the same set of properties are available
each week, but the prospective vacationers differ from week to week.
Because preferences $\pref_i$ are private to a user,
we instead assume access to {\em user features}, $u_i \in \R^{d'}$,
which are informative of $\pref_i$ in some way.
%
Letting $U \in \R^{n \times d'}$ denote the set of all user features,
each market is thus defined by a tuple $M=(U,X,p)$
of users, items, and prices.



We assume access to a sample set
$\smplst=\{(M^{(\ell)},y^{(\ell)})\}_{\ell=1}^L$
of markets $M_\ell=(U^{(\ell)},X,p^{(\ell)}) \sim \dist$ and corresponding user choices $y^{(\ell)}$.
Note this permits item prices to vary across samples, 
i.e., $p^{(\ell)}$ can be specific to the set of users $U^{(\ell)}$.
Our overall goal will be to use $\smplst$ 
to learn representations that entail useful decongested allocations, as illustrated in Figure~\ref{fig:illustration}.
Concretely, we aim for optimizing the \emph{expected welfare} induced by allocations,  
i.e., the expected sum of values of allocated items:
%
\begin{equation}
\label{eq:learning_objective}
\welf_\dist(\mask) = 
\expect{\dist}{\sum\nolimits_{ij} a(\mask)_{ij} v_{ij} }, \quad
a(\mask) = \alloc(y_1,\dots,y_n), \quad
y_i =
\choice(X,p;v_i,\mask)
\end{equation}
where
expectation is taken also w.r.t. to possible randomization in $\alloc$
and $\choice$.
Thus, we wish to solve $\argmax_\mask W_\dist(\mask)$.
Importantly, note that while choices are made based on perceived values $\pval$,
as shaped by $\mask$,
welfare itself is computed on the basis of true values $v$---which are unobserved.

\section{A Differentiable Proxy for Welfare} \label{sec:method}
We now turn to describing our approach for learning useful decongesting representations.

\extended{
The fundamental challenge in optimizing welfare in our setting is that
partial information can restrict the ability of prices to decongest.
We begin with the classic problem of \emph{optimal assignment},
in which setting prices is an effective means for decongesting a market.
We then make the connection to Eq.~\eqref{eq:learning_objective},
and discuss how representation can increase congestion by distorting the perception of value.
Building on this, we then proceed to present our approach for learning decongesting representations.\looseness=-1
}


\begin{figure}
    \centering
    \includegraphics[width=\linewidth]{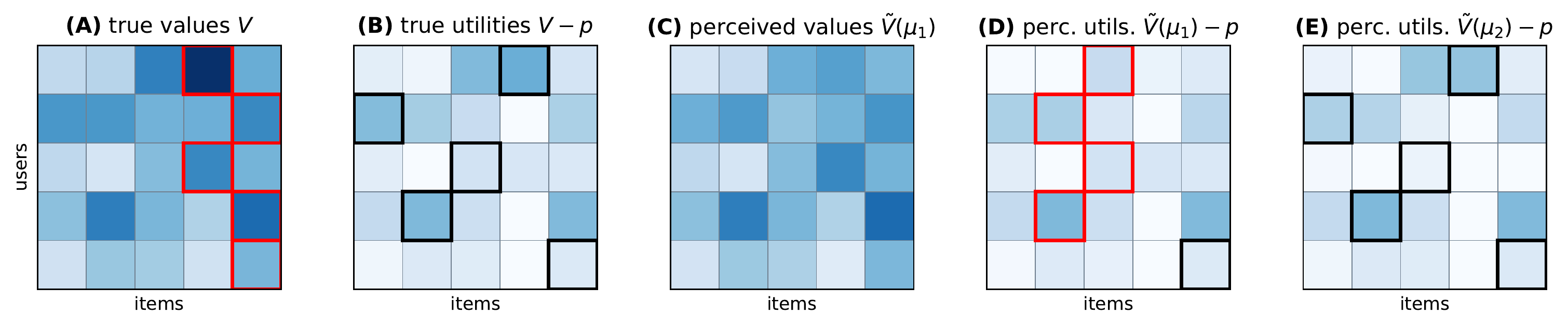}
    \caption{
    Values, prices, and choices.
    \textbf{(A)} A matrix $V$ of user-item values $v_{ij}$ in a market. 
    User choices naturally congest (red squares),
    but at full information can be decongested with prices (\textbf{(B)}; black squares).
    Partial information may distort values only mildly (\textbf{(C)}, vs. (A)),
    but still deem prices as ineffective \textbf{(D)}.
    Nonetheless, some representations are better than others---and those we seek \textbf{(E)}.\looseness=-1
    }
    \label{fig:illustration}
\end{figure}

\paragraph{Welfare decomposition.}
The main difficulty in optimizing Eq.~\eqref{eq:learning_objective}
is that we do not have access to true valuations.
To remove the reliance on $v$, our first step is to decompose welfare into two terms.
Let $\welf_M=\sum_{ij} \abar_{ij} v_{ij}$ be the expected welfare for a single market $M$, where $\abar_{ij}=\frac{1}{n_j} y_{ij}$ 
denote expected allocations
with $n_j=\sum_i y_{ij}$,
and defining $\abar_{ij}=0$ when $n_j=0$.
We can rewrite $\welf_M$ as:
\begin{equation}
\label{eq:welfare_decomposition}
\welf_M =
\sum_{ij}\tfrac{1}{n_j}y_{ij}v_{ij} =
\sum_j(1-1+\tfrac{1}{n_j})\sum_i y_{ij} v_{ij} =
\underbrace{\sum\nolimits_{ij} y_{ij} v_{ij}}_{\text{(I)}}
 + \underbrace{\sum\nolimits_j(\tfrac{1}{n_j}-1)\sum\nolimits_i y_{ij}v_{ij}}_{\text{(II)}}    
\end{equation}
In Eq.~\eqref{eq:welfare_decomposition},
term (I) encodes the value users would have gotten from their choices---had there been no supply constraints.
Term (II) then corrects for this, and appropriately penalizes excessive  allocations.\footnote{Note that no penalty is incurred if an item is chosen by at most one user, since either $\frac{1}{n_j}-1=0$ or $y_{ij}=0$.}

\paragraph{Proxy welfare.}
Absent the $v_{ij}$,
a natural next step is to replace Eq.~\eqref{eq:welfare_decomposition} with a tractable lower bound proxy.
For term (I), note that if $y_{ij}=1$ then $\pval_{ij}>p_j$ (Eq.~\eqref{eq:percieved_values_and_choices}),
and since $\pref,x \ge 0$, it also holds that $v_{ij}\ge \pval_{ij}$
(since masking can only decrease perceived value).
Hence, we can replace $v_{ij}$ with $p_j$.
For term (II), since $\frac{1}{n_j}-1\le 0$,
and since we assume $v\le 1$,
using $n_j = \sum_i y_{ij}$
we can write:\looseness=-1
\begin{equation}
\label{eq:welfare_lower_bound}
\welfprxy_M =
\underbrace{\sum\nolimits_{ij} y_{ij} p_j}_{=\,\selection(y,p)}
-\underbrace{\sum\nolimits_j \max\{0,n_j-1\}}_{=\,\decongestion(y)}
\le \welf_M
\end{equation}
which removes the explicit dependence on values,
and relies only on choices.
The two terms in $\welfprxy_M$ can now be interpreted as:
(I) \emph{selection}, which expresses the total market value of users' choices,
as encoded by prices;
and (II) \emph{decongestion}, which penalizes excess demand per item.
Notice that $n-\decongestion(y)$ is simply the number of allocated items, $|\alloc(y)|$.
To extend beyond unit-supply, we can replace $n_j-1$ with a more general $n_j-c_j$ when there are $c_j$ copies of item $j$.

\begin{figure}[t!]
    \centering
    \includegraphics[width=0.98\linewidth]{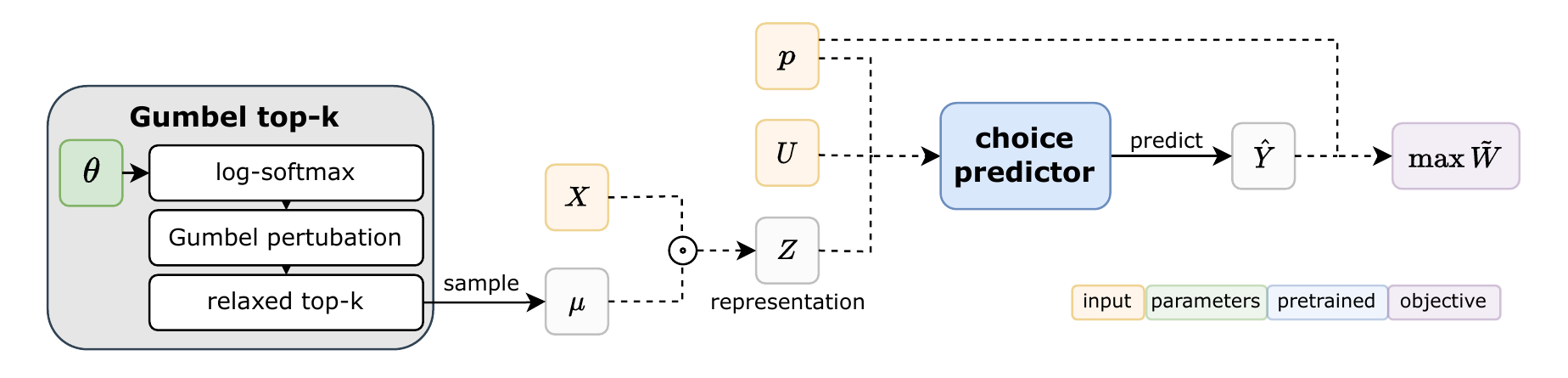} 
    \caption{
    A schematic illustration of our proposed differentiable learning framework.
    \extended{Illustration of our proposed method. $X$ is items' features matrix, $U$ is users features matrix. $\mu$ is a soft mask sampled from the categorical distribution $softmax(\Theta)$, where $\Theta$ are the trainable parameters of the model. $X$ is then masked (element-wise multiplied; $\odot $) by $\mu$ to obtain $Z$. The choice prediction model is trained separately on choices $y$ made based on $\mu \sim \pi_0$ given for the train set. The objective function is to maximize our proxy welfare $\tilde{W}$. 
    }
    }
    \label{fig:architecture}
\end{figure}



Eq.~\eqref{eq:welfare_lower_bound} still depends on values implicitly through choices $y$.
Our next step is to replace these with 
\emph{predicted choices}, $\yhat_i(\mask) = f(X,p;u_i,\mask)$,
where $f$ is a predictive model pretrained on choice data in $\smplst$:\looseness=-1
\begin{equation}
\label{eq:prediction_objective}
\fhat = \argmin_{f \in F} \sum\nolimits_{(M,y) \in \smplst} \sum\nolimits_{i \in [n]}
\Loss(y_i, f(X,p;u_i,\mask))
\end{equation}
for some model class $F$ and loss function $\Loss$ (e.g., cross-entropy),
which decomposes over users. 
%
%
%
%
Plugging the learned $f$ into
Eq.~\eqref{eq:welfare_lower_bound}
and averaging over markets in $\smplst$
obtains our empirical proxy objective:\looseness=-1
\begin{equation}
\label{eq:proxy_mask}
\welfprxy_\smplst(\mask) = 
\frac{1}{N}
\sum\nolimits_{M \in S}
\Big[
\sum\nolimits_{ij} \yhat_{ij}(\mask) p_j 
- \sum\nolimits_j \max \! \big\{0,\nhat_j(\mask) - 1 \big\} 
\Big]
\end{equation}
where $\nhat_j(\mask)=\sum_i \yhat_{ij}(\mask)$.
We interpret this as follows:
In principle, 
Eq.~\eqref{eq:proxy_mask}
seeks representations $\mask$ that entail low congestion by optimizing the $\decongestion$ term;
however, since there can be many decongesting solutions,
the additional $\selection$ term reguralizes learning towards good solutions.\looseness=-1

\paragraph{Differentiable proxy welfare.}
One challenge
in optimizing Eq.~\eqref{eq:proxy_mask}
is that both predicted choices $\yhat$ and masks $\mask$ are discrete objects.
To enable end-to-end learning, we replace these with differentiable surrogates.
For $\yhat$, we substitute `hard' argmax predictions
with `soft' predictions $\ybar_i(\mask)$ using softmax.
For masks, instead of optimizing over individual (discrete) masks,
we propose to learn masking \emph{distributions}, 
$\pi_\theta$,
that are differentiable in their parameters $\theta$.
A natural choice in this case is the multinomial distribution,
where $\theta \in \R^d$ assigns weight $\theta_r$ to each feature $r \in [d]$,
and masks are constructed by drawing $k$ features sequentially without replacement
in proportion to (re)normalized weights, $r \sim \softmax_\masktemp(\theta)$,
where $\masktemp$ is a temperature hyper-parameter.
Our final differentiable proxy objective is:\looseness=-1
\begin{equation}
\label{eq:proxy_differentiable}
\thetahat = \argmax\nolimits_{\theta \in \R^d} \welfprxy_\smplst(\pi_\theta),
\quad\,\,\, \text{where} \,\,\,\,\,\,
\welfprxy_\smplst(\pi_\theta) = \expect{\mask \sim \pi_\theta}{\welfprxy_\smplst(\mask)}
\end{equation}
%
To solve Eq.~\eqref{eq:proxy_differentiable}, we make use of
the \emph{Gumbel top-$k$ trick} \citep{vieira2014gumbel,jang2017categorical}:
by reparametrizing $\pi_\theta$, variation in masks due to $\theta$
is replaced with random perturbations $\varepsilon$;
this separates $\theta$ from the sampling process, which then permits to pass gradients effectively.
We then use the method from \citep{xie2019subsets} to smooth the selection of the top-$k$ elements.
For the forward step, the expectation in Eq.~\eqref{eq:proxy_differentiable} is approximated by computing an average over samples $\mask \sim \pi_\theta$. 
Once $\thetahat$ has been learned, at test time we can either sample from $\pi_\thetahat$ as a masking policy, or commit to $\mask_\thetahat$, defined to include the $k$ largest entries in $\thetahat$.
See Figure~\ref{fig:architecture} for an illustration of the different components of our proposed framework.


\paragraph{Practical considerations.}
One artifact of 
transitioning from Eq.~\eqref{eq:welfare_decomposition} to Eq.~\eqref{eq:welfare_lower_bound}
is that the different terms may now 
become unbalanced in terms of scale.
As a remedy, we propose to reweigh them as
$(1-\lambda) \cdotp \selection + \lambda \cdotp \decongestion$,
where $\lambda$ is a hyper-parameter that can be tuned via experimentation;
practically, our empirical analysis suggests that learning is fairly robust to the choice of $\lambda$.
In addition, we have also found it useful to add a penalty on non-choices,
i.e., $-\sum_i \one{\yhat_{i}=0}$, also weighted by $\lambda$.
This can be interpreted as also reducing congestion on the `no-choice' item, 
and as accentuating the reward of choosing real items (since no choice gives zero utility; see Appx.~\ref{appendix: no-choice penalty}).
\section{Theoretical Analysis} \label{sec:theory}

The core of our approach relies on minimizing congestion as a proxy to maximizing welfare. It is therefore natural to ask: when does decongestion improve welfare?  Focusing on an individual market, in this section we give simple conditions under which allocating more items guarantees an improvement in welfare.
\niradd{Here we consider $p$ to be competitive-equilibrium (CE) prices of the market under full information,
meaning that under full information, every
item with a strictly positive price is sold
and every user can be allocated an item in their demand set.
Proofs are deferred to Appendix~\ref{apx:analysis}.}

We start from the strongest type of relation between congestion and welfare, in which allocating more items is always better, irrespective of which items and to which users.
\begin{definition} \label{def:expost} A market with valuations $v_{ij}$ is \emph{\textbf{congestion monotone}} 
if for all $s \in [m]$,
any allocation of $s$ items gives (weakly) better welfare than any allocation of $s' < s$  items.
\end{definition}
Our first result shows that monotonicity holds in economies in which users' valuations for the items are close, as expressed in the following sufficient condition.
\begin{proposition}
\label{prop:ex_post_monotone}
In a market with $n$ users, $m$ items, and valuations $v_{ij}$, denote $\vmin = \min_{ij} v_{ij}$ and $\vmax = \max_{ij} v_{ij}$.
If $\frac{\vmax - \vmin}{\vmin} \le \frac{1}{m-1}$, then the market is congestion monotone. 
\end{proposition}

Such monotonicity 
provides us with very strong guarantees: it will sustain under any user behavior, allocation rule, and randomized outcome. However, this property
is demanding in that it considers \emph{all} allocations---whereas some allocations may not be admissible,
i.e., result from users choosing on the basis of some representation.
We now proceed to pursue this case.

\begin{definition}[Admissible allocation]
An allocation $a$ is  {\em \textbf{admissible}}, denoted
$\tilde{a}$, if agents are only assigned 
their best-response items 
defined with   respect to perceived 
values $\tilde{v}$ at prices $p$. 
\end{definition}


\begin{definition}[Restricted optimality]
An allocation $a$ is  {\em \textbf{restricted optimal}} if $a$  is  welfare-optimal at true valuations $v$ in the  economy $E=(G_a,N_a)$, where $G_a$ and  $N_a$ denote the items and agents, respectively, that are allocated; 
i.e., the economy restricted to the items and agents that are allocated.  
\end{definition}

This  property,
which in effect defines optimality on a restricted economy,
can be established through a set of sufficient conditions by reasoning with suitable notions of competitive equilibrium that arise when working with admissible allocations.
To  model the way we handle congestion, let $A$ denote a {\em randomized allocation}, 
with a product structure defined as follows.
Let $G(A)$ denote the set of items allocated.\footnote{As explained in Section \ref{sec:setup}, throughout the paper we consider unique best responses for the users. } 
The product structure requires that 
for each item $j\in G(A)$, some set $N_j$ of agents compete
 for $j$ 
 with $N_j\cap N_{j'}=\emptyset$, for all $j\neq j'$.
Each agent  $i\in N_j$ is allocated item $j$ uniformly at random,
 so that $\Pr_A[i]=1/|N_j|$ is the probability that $i$ is allocated.
 %
%
%
We say that a  randomized allocation $A$ is admissible  if it is a distribution over admissible allocations, and restricted optimal if it is a distribution over   restricted optimal allocations.
%
 Define $W(A)$ as the expected total welfare at true values, considering the distribution over allocations. 
We say that a randomized allocation $B$ 
{\em  extends} $A$ if $G(B)\supset G(A)$  
and $\Pr_B[i]\geq \Pr_A[i]$
for  all agents $i\in [n]$ 
(i.e., no agent faces more congestion).
%
%
\begin{theorem} 
Given two randomized allocations, $A$ and $B$,  where $B$ extends $A$
 and $B$ is restricted optimal,
then $W(B)\geq W(A)$, with $W(B)>W(A)$ if $v_{ij}>0$ for all $i$, $j$.
 \label{thm:welfare}
\end{theorem}

The main idea behind this result is that,
together with the  extension property,
and in a way that carefully handles randomization,
restricted optimality provides an ordering on welfare.
%

We  now seek conditions under which an admissible allocation is restricted optimal:
If these conditions hold for any admissible allocation 
in the support of a randomized allocation $B$,  then
by Thm.~\ref{thm:welfare}, 
$B$ improves welfare relative to all randomized allocations which it extends.
%
We parametrize these conditions by the margin of an admissible allocation,
which is defined as follows. 
\begin{definition}[Margin]
Let $\tilde{a}$ be an admissible allocation with allocated items and agents $\tilde{G}$ and $\tilde{N}$, resp.
Then the \emph{\textbf{margin}} of $\tilde{a}$
is the maximal $\Delta{\geq}0$ s.t.  
$\tilde{v}_{i\tilde{a}_i}-p_{\tilde{a}_i} \geq \max_{j\neq \tilde{a}_i, j\in \tilde{G}}[\tilde{v}_{ij}-p_j]+\Delta, \,\, \forall i\in \tilde{N}$.\looseness=-1
 \end{definition}

Denote agent $i$'s {\em hidden valuation} given mask $\mu$ as $v^H_{ij}=v_{ij}-\tilde{v}_{ij}$.\footnote{Here we assume w.l.o.g. (given that $0 \le v \le 1$)
 that 
 $\beta\in[0,1]^m$ and 
 $x\in [0,1]^m$.}
Each of the following conditions is sufficient for restricted optimality and thus the improving welfare claim of Theorem~\ref{thm:welfare}:
 
\begin{itemize}[leftmargin=1em, topsep=0pt]
\item 
\textbf{Condition 1: Item heterogeneity is captured in revealed features.}
A first property, sufficient for restricted optimality,
is that items $\tilde{G}$ allocated in admissible allocation $\tilde{a}$ have similar hidden features, with
$|(1-\mu)\odot (x_j-x_{j'})|_1\leq \Delta,\quad \forall j,j'\in \tilde{G}$, where $\Delta$ is the margin of the admissible allocation, $\mu$ is the mask, and $x_j$ and $x_{j'}$ the features of allocated items $j$ and $j'$, respectively.

\item
\textbf{Condition 2: Agent indifference to  hidden features.}
A second property is that  the agents $\tilde{N}$ allocated in admissible allocation $\tilde{a}$ have relatively low preference intensity for hidden features, with
$|(1-\mu)\odot \beta_i|_1\leq \Delta, \quad \forall i\in \tilde{N}$.

\item
\textbf{Condition 3: Top-item value consistency and low price variation.}
A third property relies on the item that is most preferred to
an agent  considering revealed features 
also being, approximately, the most preferred  considering
hidden features. 
In particular,
we require  
(1) {\em top-item  value consistency}, so that 
if item  $j$ satisfies $\tilde{v}_{ij}\geq \max_{j'\in \tilde{G}}\tilde{v}_{ij'}$, $\forall i \in \tilde{N}$
(i.e., it is top for $i$ considering revealed features), 
then $v^H_{ij}+\Delta \geq \max_{j'\in \tilde{G}}v^H_{ij'}$ 
(i.e., it is approximately top for $i$ considering hidden features);
and (2) {\em small price variation}, so that     $|p_j-p_{j'}|\leq \Delta$, for all items $j, j' \in \tilde{G}$.

\item
\textbf{Condition 4: Items have small hidden features.}
A fourth property that suffices for restricted optimality is that items  have small hidden features, with $|(1-\mu)\odot x_j|_1\leq \Delta,\quad \forall j\in \tilde{G}$.

\item
\textbf{Condition 5: Agent preference heterogeneity is captured in revealed features.}
A fifth property is that the agents $\tilde{N}$ allocated in addmisible allocation $\tilde{a}$ have similar preferences for hidden features, with 
$|(1-\mu)\odot(\beta_i-\beta_{i'})|_1\leq \Delta, \quad \forall i, i'\in \tilde{N}$.
\end{itemize}

\section{Experiments} \label{sec:experiments}

\subsection{Synthetic data} \label{sec:exp_synth}

We first make use of synthetic data to empirically explore our setting and approach.
Our main aim is to understand the importance of each step in our construction in Sec.~\ref{sec:method}. 
Towards this, here we abstract away optimizational and statistical issues by
focusing on small individual markets for which we can enumerate all possible masks,
and assuming access to fully accurate predictions $\yhat(\mask)=y(\mask)$.
The following experiments use $n=m=8$, $d=14$, $k=6$,
and CE prices, 
with results averaged over 10 random instances.
Additional results for an alternative decision model can be found in Appendix~\ref{apx:synth-additional}.\looseness=-1 



\paragraph{Variation in preferences.} 
In general, congestion occurs when users have similar preferences,
and our first experiment studies how the degree of preference similarity affects decongestion and welfare.
Let $V_{\text{het}}, V_{\text{hom}} \in \R^{n \times m}$ be value matrices encoding fully-heterogeneous and fully-homogeneous preferences, respectively.
We create `mixture markets' as follows:
First, we sample random item features $X$.
Then, for each of the above $V_{(i)}$, we extract user preferences $B_{(i)}$
by solving $\min_{B \ge 0} \|B X^\top - V_{(i)}\|_2$.
Finally,
for $\alpha \in [0,1]$,
we set $B_\alpha = (1-\alpha) B_{\text{het}} + \alpha B_{\text{hom}}$ to get $V_\alpha = B_\alpha X^\top$.
Thus, by varying $\alpha$, we can control the degree of preference similarity.\looseness=-1

Fig.~\ref{fig:synth} (left) presents welfare obtained by the optimal masks for the following objectives:
(i) a welfare oracle (having access to $v$),
(ii) a predictive oracle (maximizing $\yhat_{ij}(\mask) v_{ij}$ per user),
(iii) selection,
(iv) decongestion,
(v) the welfare lower bound in Eq.~\eqref{eq:welfare_lower_bound}
(namely selection minus decongestion),
and (vi) our proxy objective in Eq.~\eqref{eq:proxy_mask}.
As expected, the general trend is that less heterogeneity entails lower attainable welfare.
Prediction and selection, which consider only demand (and not supply) do not fair well, especially for larger $\alpha$.
As a general strategy, decongestion appears to be effective;
the crux is that there can be many optimally-decongesting solutions---of which some may entail very low welfare
(see subplot showing results for all $k$-sized masks in a single market).
Of these, our proxy objective encourages a decongesting solution that has also high value; results show its performances closely matches the oracle upper bound,
despite using $p$ instead of $v$ as in the welfare lower-bound.\looseness=-1

\paragraph{Perceptive distortion.}
Partial information can decrease welfare if it causes preferences to shift. 
This becomes more pronounced if preference shift increases homogeneity, which leads to increased congestion.
Since what may cause preferences to shift is the perceptive distortion of values, it would seem plausible 
to seek representations that minimize distortion.
This is demonstrated empirically in Fig. \ref{fig:synth} (right).
The figure shows evident anti-correlation between perceptive distortion
(measured as $\frac{1}{m}\|\ptilde-p\|_1$)
and welfare across al $k$-sized masks
(here we set $\alpha=0.2$).
A similar anti-correlative pattern appears in relation to preference homogeneity from perceived values
(measured using Kendall's coefficient of concordance),
suggesting that masks are useful if they entail heterogeneous choices.

\paragraph{Value dispersion.}
Although heterogeneity is important, it may not be sufficient.
As noted, markets 
with smaller margins
should make our method more susceptible to perceptive distortion.
To explore this, we study the effects of `contracting' the higher-value regime of $v$, 
achieved by taking powers $\rho < 1$ of $v$
(since $v \in [0,1]$, we have $v \le v^\rho \le 0$).
Fig. \ref{fig:synth} (center) shows results for
decreasingly smaller powers $\rho$.
As expected, since smaller $\rho$ generally increase values,
overall potential welfare increases as well.
However, as values become `tighter',
this negatively impacts the effectiveness of our approach.\looseness=-1

\begin{figure}[t!]
    \centering
    \includegraphics[width=0.388\linewidth]{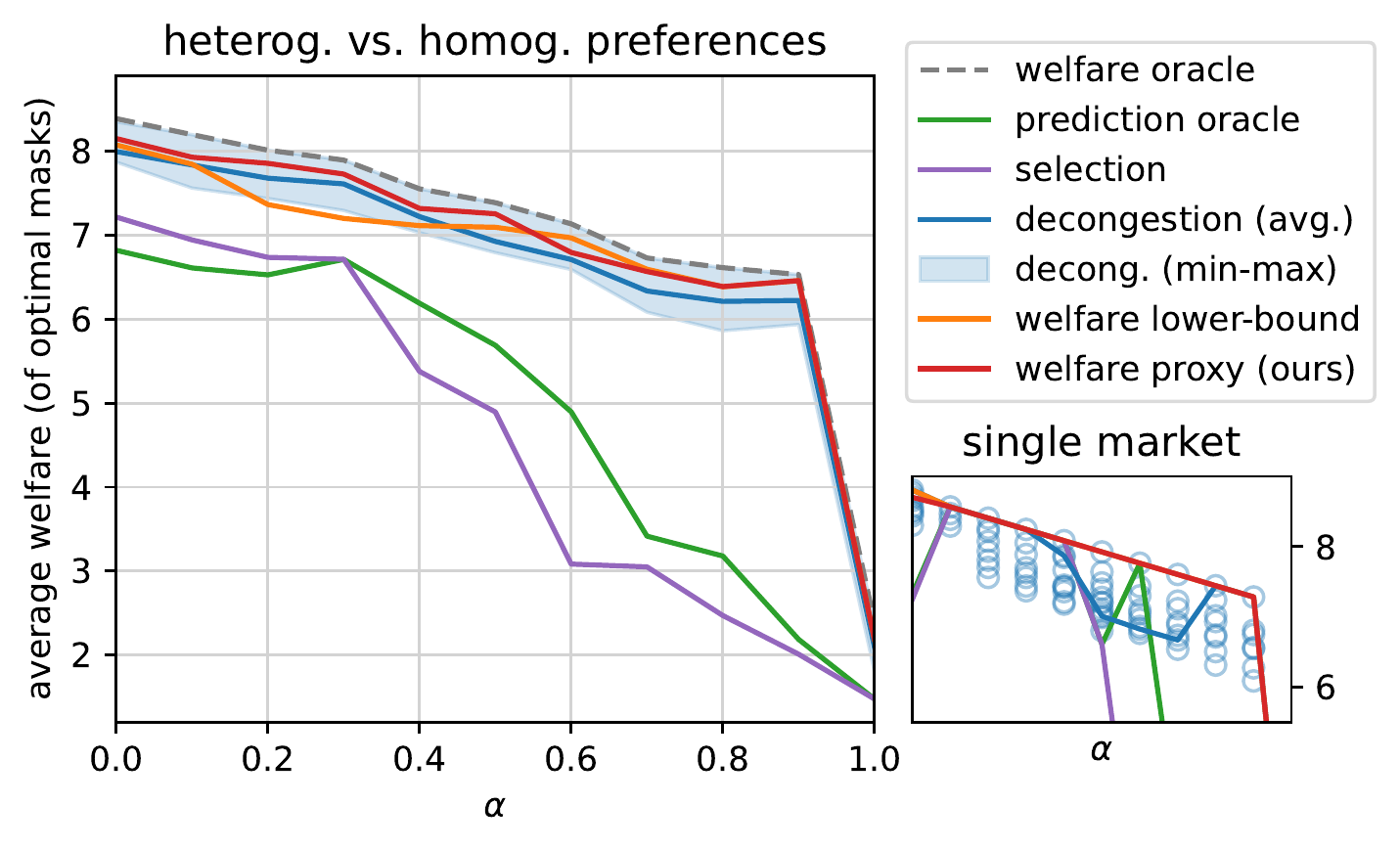} 
    \includegraphics[width=0.268\linewidth]{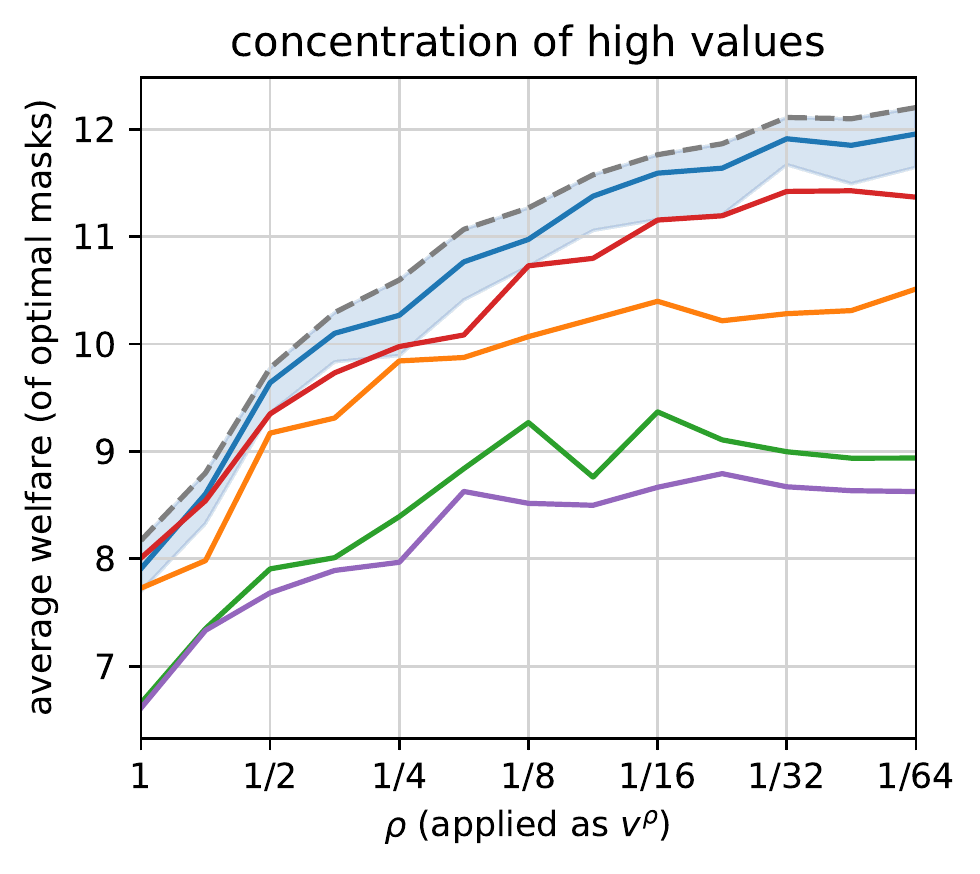}
    \includegraphics[width=0.33\linewidth]{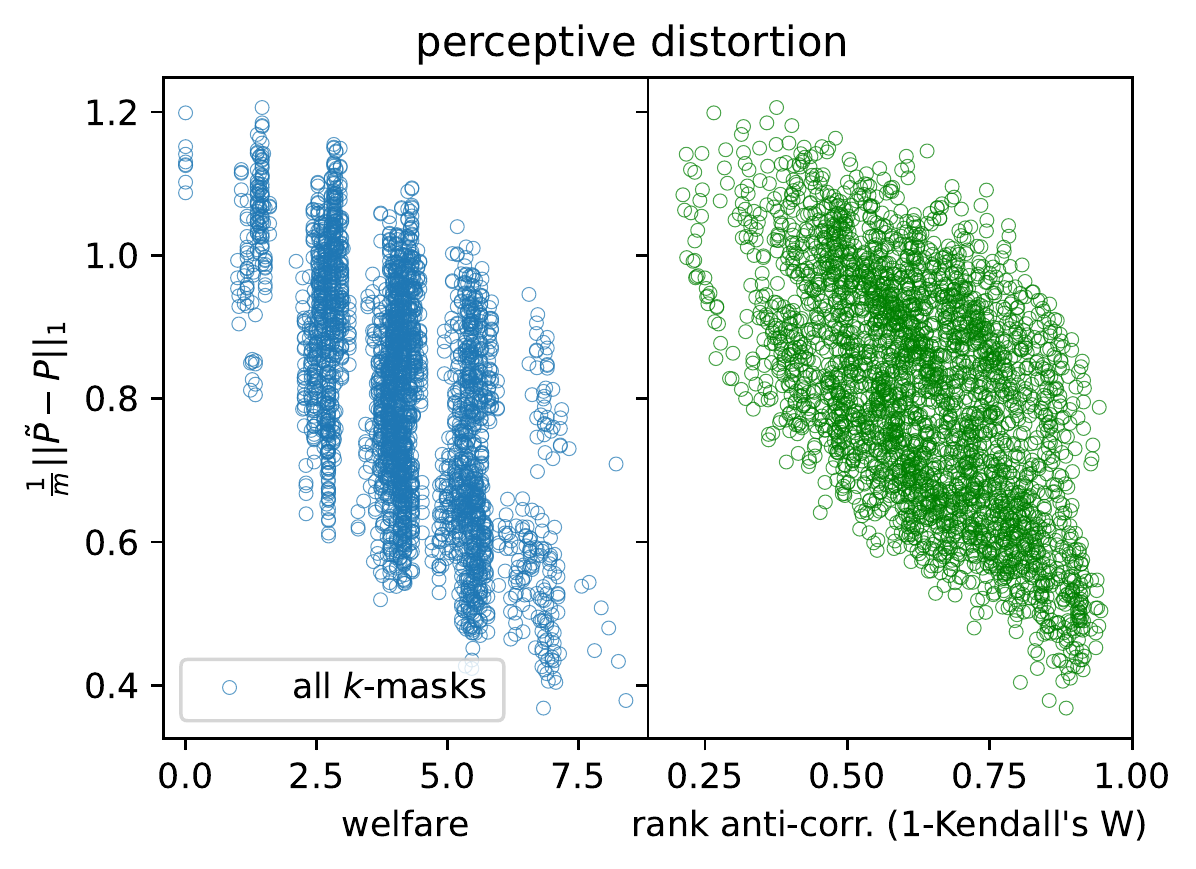}
    \caption{Results on synthetic data.
    \textbf{(Left:)} Welfare obtained by the optimal mask for different objectives, on average and for a single market (inlay).
    \textbf{(Center:)} Performance for increasingly smaller valuation gaps.
    \textbf{(Right:)} Relations between distorted values, welfare, and preference heterogeneity.
    }
    \label{fig:synth}
\end{figure}

\subsection{Real data} \label{sec:exp_real}

We now turn to experiments on real data and simulated user behavior.
\niradd{We use two datasets: MovieLens, which we present here;
and Yelp, which exhibits similar trends, and hence deferred to Appx.~\ref{apx:real-yelp}.}

\paragraph{Data.}
We use the \emph{Movielens-100k} dataset 
\cite{harper2015movielens}, which contains 100,000 movie ratings from 1,000 users and for 1,700 movies,
and is publicly-available. 
Item features $X$ and users preferences $B$ (dimension $d$) were obtained by applying non-negative matrix factorization to the partial rating matrix.
User features $U$ (dimension $d'$) were then extracted by additionally factorizing preferences $B$ as $U T^\top \approx B$,
where the inferred $T$ can be thought of as an approximate mapping from features to preferences.
We experiment in two latent dimension settings:
\emph{small} ($d=12$), which permits computing oracle baselines by enumeration; and \emph{large} ($d=100$).
In both we set $d'=d/2$.

\paragraph{Setup.}
To generate a dataset of markets $\smplst$, we first sample $m=20$ items uniformly from $X$, and then sample $L=240$ sets of $n=20$ users uniformly from $U$.
Masks $\mask$ are sampled according to a `default' masking policy $\pi_0$ 
that elicits feature importance from prices,
but ensures full support (see `price predictive' baseline below).
For prices $p$ we mainly use CE prices 
computed per market,
but also consider other pricing schemes.
Choices $y$ are then simulated as in Eq.~\eqref{eq:perceived_choice}.
Given $\smplst$, we use a 6-fold split to form different partitions into train test sets. Results are then averaged over 6 random sample sets and 6 splits per sample set (total 36, 95\% standard error bars included).

\paragraph{Method.}
For our method of decongestion by representation (\DbR),
we optimize Eq.~\eqref{eq:proxy_differentiable} using Adam \citep{adam2015} with 0.01 learning rate and for a fixed number of 300 epochs.
When $k>d/2$, we have found it useful to set $k \gets d-k$ and learn `inverted' masks $1-\mask$.
For $\lambda$, our main results use $\lambda=1-\frac{k}{2d}$, with the idea that smaller $k$ require more effort placed on decongestion, but note that this very closely matches performance for $\lambda=0.5$, 
and that results are fairly robust across $\lambda$ (see Appendix \ref{sec:appendix_lambda}).
For $f$ in Eq.~\eqref{eq:prediction_objective} we train a bi-linear model (in $u$ and $x$) for 150 epochs using cross-entropy. 
We consider three variants of our approach that differ in their test-time usage:
(i) \DbRpolicy, which samples masks from the learned policy $\mask \sim \pihat$;
(ii) \DbRmask, which commits to a single sampled mask $\hat{\mask} \sim \pihat$
(having the lowest objective value);
and (iii) \DbRtopk, which constructs and uses a mask $\mask_\thetahat$ composed of the top-$k$ entries in the learned $\thetahat$. 
For further details on implementation and optimization see Appendix \ref{sec:appendix_real_data_detailes}.
\looseness=-1

\paragraph{Baselines.}
We 
compare 
the above 
to:
(iv) \pricepred, a prediction-based method that uses the top-$k$ most informative features for predicting prices from item features, with the idea that these should also be most informative of values;
(v) \choicepred, which aims to recover the top-$k$ most important features for users by eliciting an estimate of $T$
(and hence of preferences $\beta$) from the learned choice-prediction model $f$;
(vi) an \oracle\ benchmark that optimizes welfare directly (when applicable);
and (vii) a \random\ benchmark reporting average performance over randomly-sampled $k$-sized masks.
\looseness=-1

\paragraph{Results.}
Figure~\ref{fig:real} (left, center) shows results for increasing values of $k$.
Because overall welfare quickly increases with $k$ for all methods,
for an effective comparison across $k$ we plot
the relative gain in welfare compared to \random,
with absolute values depicted within.
For the $d=12$ setting (left), results show that our approach is able to learn effective representations attaining welfare that is close to \oracle. Relative gains increase with $k$ and peak at around $k=8$.
Prediction-based methods generally improve with $k$, but at a low rate.
The inlaid plot shows a tight connection to the number of allocated items, suggesting the importance of (de)congestion in promoting welfare (or failing to do so).
For $d=100$ (center), performance of our approach steadily increase with $k$.
Here \choicepred\ preforms reasonably well for $k \approx 50$, but not so for large $k$, nor for small $k$, where \pricepred\ also fails.\looseness=-1

\paragraph{The role of prices.} \label{sec:prices_experiments}
Because our proxy welfare objective relies on prices for guiding decongestion (for which CE prices are especially useful),
we examine the robustness of our approach to differing pricing schemes.
Focusing on $d=12$ and $k=6$,
Figure~\ref{fig:real} (right) shows performance for 
(i) CE prices ranging from buyer-optimal (minimal) to seller-optimal (maximal),
and (ii) increasing levels of noise applied to mid-range CE prices.
Results show that overall performance degrades as prices become either higher or noisier, demonstrating the general importance of having value-reflective prices.
Nonetheless,
and despite its reliance on prices,
our approach steadily maintains performance relative to others.
Appendix~\ref{apx:real-pricing_schemes} shows similar results for additional variations on pricing schemes.

\begin{figure}[t!]
    \centering
    \includegraphics[width=\linewidth]{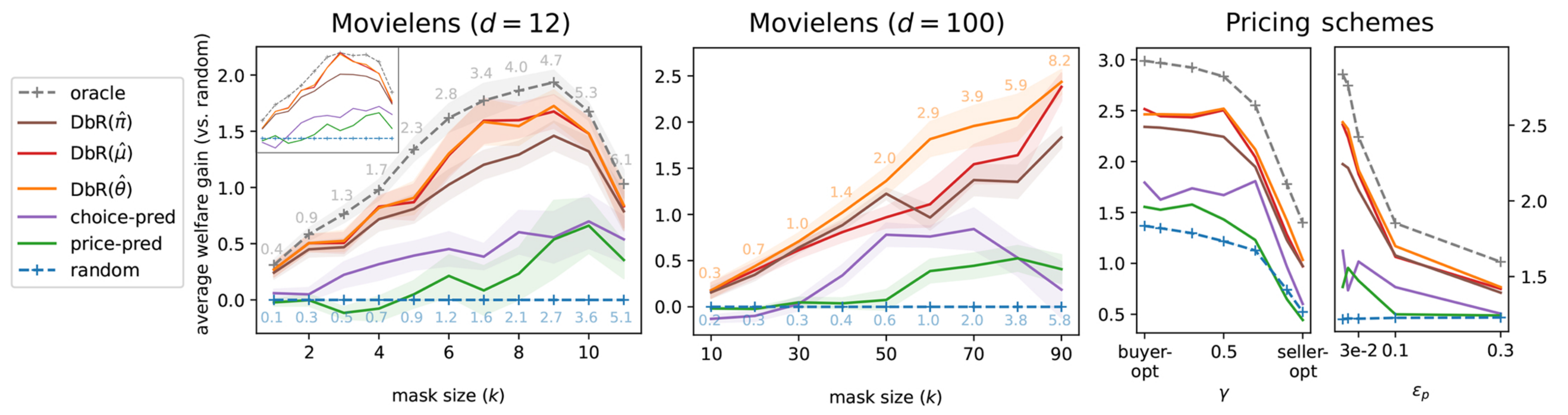}
    \caption{
    Experiments using real data.
    \textbf{(Left + center:)} Gain in welfare for increasing mask size $k$, for the Movielens dataset using $d=12$ (left) and $d=100$ (center) hidden features. Plot lines gain relative to random, numbers show absolute welfare values.
    \textbf{(Right:)} Welfare (absolute) obtained for different prices schemes:
    (i) buyer- vs. seller-optimal prices, and
    (ii) increasing additive noise.
    }
    \label{fig:real}
\end{figure}

\section{Discussion}

In this paper, we have initiated the study of decongestion by representation, 
developing a differentiable learning framework that learns item representations in order to reduce congestion and improve social welfare. 
Our main claim is that partial information is a necessary aspect of
 modern online markets, and that systems have both 
the  opportunity and responsibility in choosing representations that serve their users well.
We view our approach, which pertains to `hard' congestion  found in tangible-goods markets, and on feature-subset representations, as taking one step towards this. At the same time,   `soft' congestion,
which is prevalent in 
 digital-goods markets, also caries many adverse effects.
Moreover,   there exist various  other relevant forms of information representation (e.g., feature ranking, or even other modalities such as images or text). 
We leave these,
as well as the study of more elaborate user choice models,
as 
\galiadd{interesting directions for} future work.

\extended{
\red{
- per-item masks; parametric per-item/market masks $\mask(x)$ \\
}
}


\section*{Ethics Statement and Broader Perspectives}

Our paper considers the effect of partial information on user choices in the context of online market platforms,
and proposes that platforms utilize their control over representations to promote decongestion as a means for improving social welfare.
Our point of departure is that partial information is an inherent component of modern choice settings.
As consumers, we have come to take this reality for granted.
Still, this does not mean that we should take the system-governed decision of what information to convey about items, and how, as a  given.
Indeed, we believe it is not only in the power of platforms, but also their responsibility, to choose representations with care.
Our work suggests that `default' representations,
such as those relying on predictions of user choices,
may account for demand---but are inappropriate when
supply constraints have concrete implications on user utility.

\paragraph{Soft congestion.}
Although our focus is primarily on tangible-goods, 
we believe similar arguments hold more broadly in markets for non-tangibles, such as media, software, or other digital goods.
While technically such markets are not susceptible to `hard' congestion since there is no physical limitation on the number of item copies that can be allocated,
still there is ample evidence of `softer' forms of congestion
which similarly lend to negative outcomes.
For example, digital marketplaces are known to exhibit hyper-popularization,
arguably as the product of rich-get-richer dynamics,
and which results in strong inequity across suppliers and sellers.
Some recent works have considered the negative impact of such soft congestion, but mostly in the context of recommender systems;
we believe our conclusions on the role of representations apply also 
to `soft' congestion, perhaps in a more subtle form,
but nonetheless carrying the same important implications for welfare.
\extended{\todo{plug citations}}


\paragraph{Limitations.}
We consider the task of decongestion by representation in a simplified 
 market setting,  including several assumptions on the environment and on user behavior.
One key assumption relates to how we model user choice (Sec.~\ref{sec:setup}).
While this can perhaps be seen as less restrictive than the standard economic assumption of rationality, 
our work considers only one form of bounded-rational behavior,
whereas in reality there could be many others
(our extended experiments in Appendix~\ref{apx:synth-additional} take one small step towards considering other behavioral assumptions).
In terms of pricing, our theoretical analysis in Sec.~\ref{sec:theory}
relies on equilibrium prices with respect to true buyer preferences,
 which may not hold in practice. Nonetheless, our experiments in Sec.~\ref{sec:experiments} and  Appendix~\ref{apx:real-pricing_schemes} on varying pricing schemes show that while CE prices are useful for our approach---they are not necessary.
Our counterexample in Sec.~\ref{apx:adaptive_prices} suggests that, in the worst case, partially-informed equilibrating prices do not `solve the problem'.
For our experiments in Sec.~\ref{sec:exp_real},
as we state and due to natural limitations,
our empirical evaluation is restricted to rely on real data but simulated user behavior.
Establishing our conclusions in realistic markets requires human-subject experiments as well as extensive field work.
We are hopeful that our current work will serve to encourage these kinds of future endeavours.

\paragraph{Ethics considerations.}
Determining representations has an immediate and direct effect on human behavior, and hence must be done with care and consideration.
Similarly to recommendation,
decongestion by representation is in essence a policy problem,
since committing to some representation at one point in time can affect, through user behavior, future outcomes.
Our empirical results in Sec.~\ref{sec:experiments} suggest that learning can work well even when the counterfactual nature of the problem is technically unaccounted for (e.g., training $f$ once at the onset on $\pi_0$, and using it throughout).
But this should not be taken to imply that learning of representations in practice can succeed while ignoring counterfactuals.
For this, we take inspiration from the field of recommender systems,
which despite its historical tendency to focus on predictive aspects of recommendations, has in recent years been placing increasing emphasis on recommendation as a policy problem, and on the 
 implications of this. 


While our focus is on `anonymous' representations,
i.e., that are fixed across items and for all users---it is important to note that the \emph{effect} of representations on users is not uniform.
This comes naturally from the fact that representations affect the perception of value, which is of course personal.
As such, representations are inherently individualized.
And while this provides power for improving welfare,
it also suggests that care must be taken to avoid discrimination on the basis of induced perceptions; e.g.,
decongesting by systematically diverting certain groups or individuals from their preferred choices. \looseness=-1

Finally, we note that while promoting welfare is our stated goal and underlies the formulation of our learning objective,
the general approach we consider can in principal be used to promote other platform objectives.
Since these may not necessarily align with user interests,
deploying our framework in any real context should be done with integrity and under transparency,
to the extent possible,
by the platform. 



\subsection*{Acknowledgements}
This research was supported by the Israel Science Foundation (grant No. 278/22),
and has received funding from the European Research Council (ERC) under the European Union’s Horizon 2020 research and innovation programme (grant agreement No 740282).
Gali Noti has also been affiliated with Harvard University and the Hebrew University of Jerusalem during this project.
We would like to thank Sophie Hilgard for her conceptual and methodological contributions to the paper in its initial stages.

\bibliography{bibliography}

\begin{thebibliography}{31}
\providecommand{\natexlab}[1]{#1}
\providecommand{\url}[1]{\texttt{#1}}
\expandafter\ifx\csname urlstyle\endcsname\relax
  \providecommand{\doi}[1]{doi: #1}\else
  \providecommand{\doi}{doi: \begingroup \urlstyle{rm}\Url}\fi

\bibitem[Akerlof(1978)]{akerlof1978market}
George~A Akerlof.
\newblock The market for “lemons”: Quality uncertainty and the market mechanism.
\newblock In \emph{Uncertainty in economics}, pp.\  235--251. Elsevier, 1978.

\bibitem[Alcobendas \& Zeithammer(2021)Alcobendas and Zeithammer]{alcobendas2021adjustment}
Miguel Alcobendas and Robert Zeithammer.
\newblock Adjustment of bidding strategies after a switch to first-price rules.
\newblock \emph{Available at SSRN 4036006}, 2021.

\bibitem[Bahar et~al.(2020)Bahar, Ben-Porat, Leyton-Brown, and Tennenholtz]{bahar2020fiduciary}
Gal Bahar, Omer Ben-Porat, Kevin Leyton-Brown, and Moshe Tennenholtz.
\newblock Fiduciary bandits.
\newblock In \emph{International Conference on Machine Learning}, pp.\  518--527. PMLR, 2020.

\bibitem[Bansal et~al.(2021)Bansal, Nushi, Kamar, Horvitz, and Weld]{bansal2021most}
Gagan Bansal, Besmira Nushi, Ece Kamar, Eric Horvitz, and Daniel~S Weld.
\newblock Is the most accurate {AI} the best teammate? {O}ptimizing {AI} for teamwork.
\newblock In \emph{Proceedings of the AAAI Conference on Artificial Intelligence}, volume~35, pp.\  11405--11414, 2021.

\bibitem[Ben-Porat \& Tennenholtz(2018)Ben-Porat and Tennenholtz]{ben2018game}
Omer Ben-Porat and Moshe Tennenholtz.
\newblock A game-theoretic approach to recommendation systems with strategic content providers.
\newblock \emph{Advances in Neural Information Processing Systems}, 31, 2018.

\bibitem[Ben-Porat \& Tennenholtz(2019)Ben-Porat and Tennenholtz]{ben2019regression}
Omer Ben-Porat and Moshe Tennenholtz.
\newblock Regression equilibrium.
\newblock In \emph{Proceedings of the 2019 ACM Conference on Economics and Computation}, pp.\  173--191, 2019.

\bibitem[Chaney et~al.(2018)Chaney, Stewart, and Engelhardt]{chaney2018algorithmic}
Allison~JB Chaney, Brandon~M Stewart, and Barbara~E Engelhardt.
\newblock How algorithmic confounding in recommendation systems increases homogeneity and decreases utility.
\newblock In \emph{Proceedings of the 12th ACM conference on recommender systems}, pp.\  224--232, 2018.

\bibitem[Dean et~al.(2020)Dean, Rich, and Recht]{dean2020recommendations}
Sarah Dean, Sarah Rich, and Benjamin Recht.
\newblock Recommendations and user agency: the reachability of collaboratively-filtered information.
\newblock In \emph{Proceedings of the 2020 Conference on Fairness, Accountability, and Transparency}, pp.\  436--445, 2020.

\bibitem[Doraszelski et~al.(2018)Doraszelski, Lewis, and Pakes]{doraszelski2018just}
Ulrich Doraszelski, Gregory Lewis, and Ariel Pakes.
\newblock Just starting out: Learning and equilibrium in a new market.
\newblock \emph{American Economic Review}, 108\penalty0 (3):\penalty0 565--615, 2018.

\bibitem[Guo et~al.(2022)Guo, Kandasamy, Gonzalez, Jordan, and Stoica]{guo2022learning}
Wenshuo Guo, Kirthevasan Kandasamy, Joseph Gonzalez, Michael Jordan, and Ion Stoica.
\newblock Learning competitive equilibria in exchange economies with bandit feedback.
\newblock In \emph{International Conference on Artificial Intelligence and Statistics}, pp.\  6200--6224. PMLR, 2022.

\bibitem[Harper \& Konstan(2015)Harper and Konstan]{harper2015movielens}
F~Maxwell Harper and Joseph~A Konstan.
\newblock The movielens datasets: History and context.
\newblock \emph{ACM transactions on interactive intelligent systems (tiis)}, 5\penalty0 (4):\penalty0 1--19, 2015.

\bibitem[Hilgard et~al.(2021)Hilgard, Rosenfeld, Banaji, Cao, and Parkes]{hilgard2021learning}
Sophie Hilgard, Nir Rosenfeld, Mahzarin~R Banaji, Jack Cao, and David Parkes.
\newblock Learning representations by humans, for humans.
\newblock In \emph{International Conference on Machine Learning}, pp.\  4227--4238. PMLR, 2021.

\bibitem[Jagadeesan et~al.(2022)Jagadeesan, Garg, and Steinhardt]{jagadeesan2022supply}
Meena Jagadeesan, Nikhil Garg, and Jacob Steinhardt.
\newblock Supply-side equilibria in recommender systems.
\newblock \emph{arXiv preprint arXiv:2206.13489}, 2022.

\bibitem[Jagadeesan et~al.(2023)Jagadeesan, Jordan, and Haghtalab]{jagadeesan2023competition}
Meena Jagadeesan, Michael~I Jordan, and Nika Haghtalab.
\newblock Competition, alignment, and equilibria in digital marketplaces.
\newblock In \emph{Proceedings of the Thirty-Fifth {AAAI} Conference on Artificial Intelligence}, 2023.

\bibitem[Jang et~al.(2017)Jang, Gu, and Poole]{jang2017categorical}
Eric Jang, Shixiang Gu, and Ben Poole.
\newblock Categorical reparameterization with gumbel-softmax.
\newblock In \emph{5th International Conference on Learning Representations, {ICLR} 2017, Toulon, France, April 24-26, 2017, Conference Track Proceedings}. OpenReview.net, 2017.

\bibitem[Kingma \& Ba(2015)Kingma and Ba]{adam2015}
Diederik~P. Kingma and Jimmy Ba.
\newblock Adam: {A} method for stochastic optimization.
\newblock In Yoshua Bengio and Yann LeCun (eds.), \emph{3rd International Conference on Learning Representations, {ICLR} 2015, San Diego, CA, USA, May 7-9, 2015, Conference Track Proceedings}, 2015.
\newblock URL \url{http://arxiv.org/abs/1412.6980}.

\bibitem[Kleinberg \& Mullainathan(2019)Kleinberg and Mullainathan]{kleinberg2019simplicity}
Jon Kleinberg and Sendhil Mullainathan.
\newblock Simplicity creates inequity: implications for fairness, stereotypes, and interpretability.
\newblock In \emph{Proceedings of the 2019 ACM Conference on Economics and Computation}, pp.\  807--808, 2019.

\bibitem[Kremer et~al.(2014)Kremer, Mansour, and Perry]{kremer2014implementing}
Ilan Kremer, Yishay Mansour, and Motty Perry.
\newblock Implementing the “wisdom of the crowd”.
\newblock \emph{Journal of Political Economy}, 122\penalty0 (5):\penalty0 988--1012, 2014.

\bibitem[Krishnaswamy et~al.(2021)Krishnaswamy, Li, Rein, Zhang, and Conitzer]{krishnaswamy2021classification}
Anilesh~K Krishnaswamy, Haoming Li, David Rein, Hanrui Zhang, and Vincent Conitzer.
\newblock Classification with strategically withheld data.
\newblock In \emph{Proceedings of the AAAI Conference on Artificial Intelligence}, volume~35, pp.\  5514--5522, 2021.

\bibitem[Mansour et~al.(2015)Mansour, Slivkins, and Syrgkanis]{mansour2015bayesian}
Yishay Mansour, Aleksandrs Slivkins, and Vasilis Syrgkanis.
\newblock Bayesian incentive-compatible bandit exploration.
\newblock In \emph{Proceedings of the Sixteenth ACM Conference on Economics and Computation}, pp.\  565--582, 2015.

\bibitem[Mladenov et~al.(2020)Mladenov, Creager, Ben-Porat, Swersky, Zemel, and Boutilier]{mladenov2020optimizing}
Martin Mladenov, Elliot Creager, Omer Ben-Porat, Kevin Swersky, Richard Zemel, and Craig Boutilier.
\newblock Optimizing long-term social welfare in recommender systems: A constrained matching approach.
\newblock In \emph{International Conference on Machine Learning}, pp.\  6987--6998. PMLR, 2020.

\bibitem[Nair et~al.(2022)Nair, Ghalme, Talgam-Cohen, and Rosenfeld]{nair2022strategic}
Vineet Nair, Ganesh Ghalme, Inbal Talgam-Cohen, and Nir Rosenfeld.
\newblock Strategic representation.
\newblock In \emph{International Conference on Machine Learning}, pp.\  16331--16352. PMLR, 2022.

\bibitem[Noti \& Chen(2023)Noti and Chen]{noti2022learning}
Gali Noti and Yiling Chen.
\newblock Learning when to advise human decision makers.
\newblock In \emph{Proceedings of the Thirty-Second International Joint Conference on Artificial Intelligence}, pp.\  3038--3048, 2023.

\bibitem[Riedl(2019)]{riedl2019human}
Mark~O Riedl.
\newblock Human-centered artificial intelligence and machine learning.
\newblock \emph{Human Behavior and Emerging Technologies}, 1\penalty0 (1):\penalty0 33--36, 2019.

\bibitem[Schmit \& Riquelme(2018)Schmit and Riquelme]{schmit2018human}
Sven Schmit and Carlos Riquelme.
\newblock Human interaction with recommendation systems.
\newblock In \emph{International Conference on Artificial Intelligence and Statistics}, pp.\  862--870. PMLR, 2018.

\bibitem[Shapley \& Shubik(1971)Shapley and Shubik]{shapley1971assignment}
Lloyd~S Shapley and Martin Shubik.
\newblock The assignment game {I}: The core.
\newblock \emph{International Journal of game theory}, 1\penalty0 (1):\penalty0 111--130, 1971.

\bibitem[Tabibian et~al.(2019)Tabibian, G{\'o}mez, De, Sch{\"o}lkopf, and Rodriguez]{tabibian2019consequential}
Behzad Tabibian, Vicen{\c{c}} G{\'o}mez, Abir De, Bernhard Sch{\"o}lkopf, and Manuel~Gomez Rodriguez.
\newblock Consequential ranking algorithms and long-term welfare.
\newblock \emph{arXiv preprint arXiv:1905.05305}, 2019.

\bibitem[Thaler \& Sunstein(2008)Thaler and Sunstein]{nudge}
Richard~H. Thaler and Cass~R. Sunstein.
\newblock \emph{Nudge}.
\newblock Yale University Press, New Haven, CT and London, 2008.
\newblock ISBN 978-0-300-12223-7.

\bibitem[Vieira(2014)]{vieira2014gumbel}
Tim Vieira.
\newblock Gumbel-max trick and weighted reservoir sampling, 2014.
\newblock URL \url{http://timvieira.github.io/blog/post/2014/08/01/gumbel-max-trick-and-weighted-reservoir-sampling/}.

\bibitem[Wu et~al.(2019)Wu, Swait, and Chen]{wu2019feature}
Fang Wu, Joffre Swait, and Yuxin Chen.
\newblock Feature-based attributes and the roles of consumers' perception bias and inference in choice.
\newblock \emph{International Journal of Research in Marketing}, 36\penalty0 (2):\penalty0 325--340, 2019.

\bibitem[Xie \& Ermon(2019)Xie and Ermon]{xie2019subsets}
Sang~Michael Xie and Stefano Ermon.
\newblock Reparameterizable subset sampling via continuous relaxations.
\newblock \emph{International Joint Conference on Artificial Intelligence (IJCAI)}, 2019.

\end{thebibliography}
\bibliographystyle{iclr2024_conference}

\clearpage

\appendix


\begin{figure}[b!]
    \centering
    \includegraphics[height=10em]{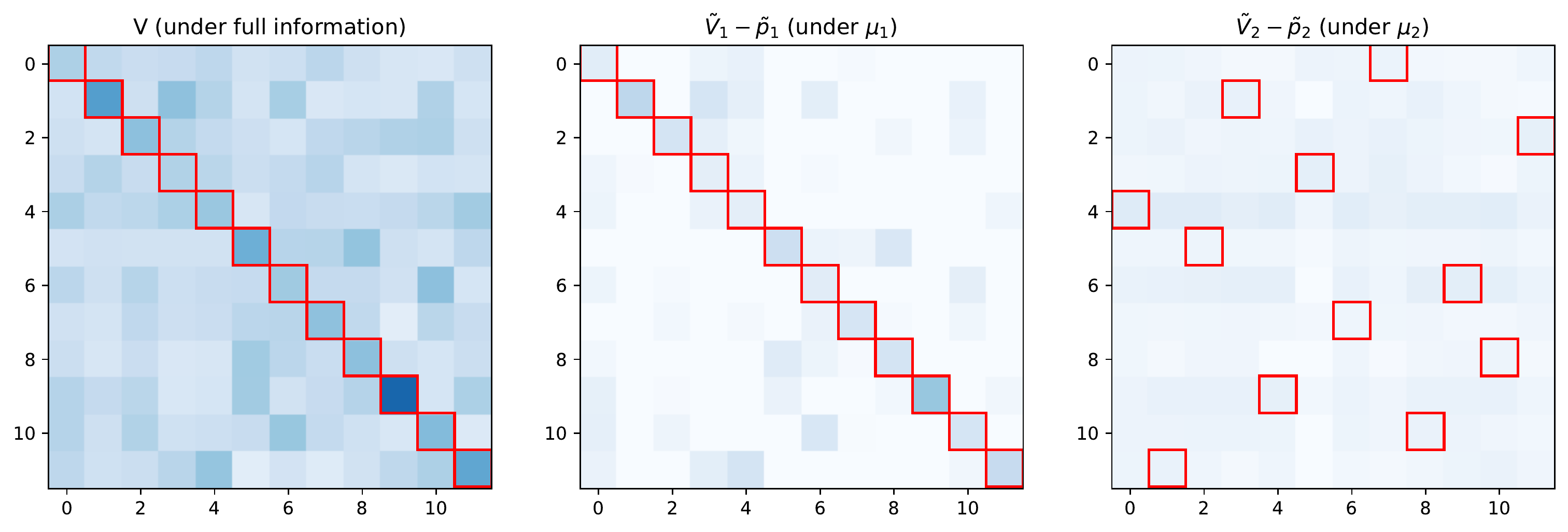} 
    \includegraphics[height=9.5em]{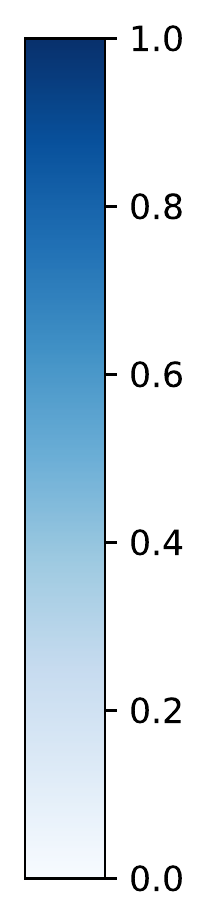} 
    \caption{
    An example market in which different representations entail very different allocations at partial-information market-clearing prices.
    \textbf{(Left:)}
    The true valuation matrix $V$, with corresponding allocations (red squares) for choices made under full information and CE prices $p^*$.
    \textbf{(Center:)}
    Perceived values $\tilde{V}_1$ under representation $\mu_1$, minus corresponding  partial-information CE prices $\tilde{p}_1$. Resulting allocations are optimal.
    \textbf{(Right:)}
    Perceived values $\tilde{V}_2$ minus partial-information prices $\tilde{p}_2$ under $\mu_2$.
    Resulting allocations are highly sub-optimal.\looseness=-1
    }
    \label{fig:counterexample}
\end{figure}

\niradd{
\section{A note on adaptive prices} \label{apx:adaptive_prices}
Our settings makes the assumption that prices are fixed.
This is motivated by settings in which sellers are slow to adapt (or do not adapt at all),
and in which representations can be adjusted to take effect more quickly.
In this sense, we see representations as adapting to prices---rather than vice versa.

\paragraph{Adaptive prices.}
An alternative would be to consider prices that adapt to revealed demand,
and in particular,
prices $\tilde{p}$ that attain  competitive equilibrium under perceived values $\tilde{v}$; i.e., ``partially-informed competitive equilibrium prices," or ``partially-informed prices."
These prices would  clear the market, but nonetheless have several significant drawbacks:

\begin{itemize}[leftmargin=1em, topsep=0pt]
\item
First, such prices would completely ignore true valuations $v$, and the actual values that users obtain from items would have no effect on the market. We find this to be unrealistic; a more plausible alternative would be to have (past) true values propagate to influence (future) prices in some manner (e.g., via users posting reviews). Fixed prices can be seen as one (indirect) way to achieve this.\looseness=-1

\item
Second, and relatedly, while partially-informed prices do solve congestion,
they do so without any guarantees on welfare;
in fact, in our setting, welfare under such prices can be arbitrarily low (see below).
This is in contrast to fully-informed prices, which simultaneously minimize congestion \emph{and} maximize welfare.\looseness=-1

\item
Third, in our setting, such partially-informed prices would likely be much lower than prices at full information. This may push sellers to leave the platform if they have an external option, or if prices fall below production costs, this reducing welfare.

\item
Fourth, and most importantly, partially-informed prices \emph{still depend on what information is revealed}, i.e., they will be different under different masking schemes. Thus, the problem of choosing what information to convey would remain and in fact become more difficult, as learning must now anticipate not only choices, but also induced prices, under possible representations.
\end{itemize}

Therefore, while learning representations for adapting prices is an intriguing direction,
we feel it is deserving of designated future work.

\paragraph{A constructive example.}
We now show how in our setting, partially-informed prices can give arbitrarily-bad welfare (in the worst case) as a result of their dependence on representations.
We prove this by constructing an example in which one representation yields approximately optimal welfare,
whereas another yields (approximately) only a small constant fraction,
under corresponding partially-informed prices.
The construction works by setting half of the features to encode most of the true values of items,
and the other half to encode noise.
The former subset corresponds to a `good' representation,
for which prices need not adapt much,
and hence preserves optimal choices.
The latter subset corresponds to a `bad' representation,
which is highly uninformative of values;
this causes prices to adapt in a way that entails a `random' decongested allocation providing very low welfare.

Figure~\ref{fig:counterexample} illustrates the values and choices under the different representations and pricing schemes for our example.
The precise numerical values used in the example can be found in our code base.

}


\section{Theoretical Analysis} \label{apx:analysis}
\subsection*{Competitive equilibrium}
Let $p=(p_1,\ldots,p_m)$ denote  item prices,  $a$ denote
 a feasible allocation (i.e., each item is allocated at most once and each user to at most one item), and $v_i$ agent $i$'s true valuation. Let $\tilde{v}_i$ denote
agent $i$'s perceived valuation 
given mask $\mu$, 
and  $v^H_{ij}=v_{ij}-\tilde{v}_{ij}$
 denote agent $i$'s {\em hidden valuation}.
We make the technical assumption that each user has a unique best response,
but note the analysis extends to demand sets that are not a singleton by heuristically selecting an item from the demand set.%
\footnote{
Alternatively, one can infinitesimally perturb the preference vectors and obtain a unique best response.
}
 
%
\begin{definition}
    $(a,p)$ is a competitive equilibrium if (1) $a_i\in\argmax_j[v_{ij}-p_{j},0]$ for all $i$, and (2) any item with $p_j>0$
 is allocated. 
 \end{definition}

Competitive equilibrium requires that allocation $a$ is (1)  a best response for each agent,  and (2) maximizes revenue. 
The following is well known, the proof is included for completeness.
\begin{theorem}
A CE is welfare optimal.
\end{theorem}
\begin{proof}
   The primal assignment problem is 
    \begin{align}
    \label{eq:ce_primal}
   \max_{a}&\sum_i\sum_j a_{ij}v_{ij} \\
   \mbox{s.t.} \quad& \sum_{i}a_{ij}\leq 1 \quad, \forall j \quad [\mbox{dual}\ p_j] \nonumber \\
   &  \sum_{j}a_{ij}\leq 1 \quad, \forall i \quad [\mbox{dual}\ \pi_i] \nonumber \\
   &x_{ij}\geq 0 \nonumber
    \end{align}

    The dual is 
    \begin{align}
    \label{eq:ce_dual}
        \min_{\pi,p} &\sum_j p_j+ \sum_i \pi_i\\
         \mbox{s.t.} \quad& \pi_i+p_{j}\geq v_{ij} \quad \forall i, j \quad\quad [\mbox{dual}\ a_{ij}] \nonumber \\
        & \pi_i,p_{j}\geq 0. \nonumber
    \end{align}

The optimality of CE $(a,p)$,  along with $\pi_i=\max_j[v_{ij}-p_{j},0]$ to complete the dual,   
is established by checking complementary slackness (CS). The primal CS condition is $a_{ij}>0 \Rightarrow \pi_i+p_j=v_{ij}$, and satisfied 
since  agent $i$ receives an item in its best response set when non-empty (CE),
 and by the construction of $\pi_i$.
 The dual CS conditions are $\pi_i>0\Rightarrow \sum_j a_{ij}=1$ and $p_j>0\Rightarrow \sum_i a_{ij}=1$, and  satisfied
 by the CE properties, 
since every agent with a non-zero demand set gets an item 
and  every item with positive price is allocated.
\end{proof}


CE prices form a lattice,  in general are not unique,  and  price the {\em core} of the assignment game \cite{shapley1971assignment}. 
Amongst the set of CE prices, the 
buyer-optimal and seller-optimal prices are especially salient.
%

\subsection*{Congestion monotonicity}
\begin{proof} (Proposition 1.): 
Let $\mathcal{A}_s$ denote the set of all feasible allocations of exactly $s$ items, such that every set $A \in \mathcal{A}_s$ is a set of user-item pairs that represents an allocation of $s$ items. Value matrix $(v_{ij})$ is congestion monotone if and only if for every $s \leq m$ it holds that 
\begin{equation*}
\max_{A \in \mathcal{A}_{s-1}} \sum_{(i,j)\in A} v_{ij} \leq
\min_{A \in \mathcal{A}_s} \sum_{(i,j)\in A} v_{ij}. 
\end{equation*} 

Next, we define $\delta_{ij} = v_{ij} - v_{min}$ and write every value in $(v_{ij})$ as $v_{ij} = v_{min} + \delta_{ij}$. Using these notations, the congestion monotonicity condition is:
\begin{equation*}
v_{min} \geq 
\Big(\max_{A \in \mathcal{A}_{s-1}} \sum_{(i,j)\in A} \delta_{ij}\Big)
- 
\Big(
\min_{A \in \mathcal{A}_s} \sum_{(i,j)\in A} \delta_{ij} 
\Big).
\end{equation*}
Since $s \leq m$ and since the last summation is of positive terms, we have that 
a sufficient condition is: 
$v_{min} \geq 
(m-1) \cdot \max(\delta_{ij})
    = (m-1) \big( v_{max} - v_{min} \big)$, as required.
\end{proof}

\subsection*{Restricted optimality}

We start by discussing deterministic allocations and then proceed to the proof of Theorem 1 and the proofs for the sufficient conditions for restricted optimality.
Let {\em welfare} $W(a)=\sum_i \sum_j a_{ij} v_{ij}$. Let   $G_a$ and  $N_a$ 
denote the items and agents, respectively, that are allocated in allocation $a$.
Say that $a$ is  {\em restricted optimal} if 
and only if  $a$  is  welfare optimal at true valuations $v$ 
in the  economy $E=(G_a,N_a)$;
 i.e., the economy restricted to the items and agents that are allocated.  
Say that an 
allocation $b$ {\em extends} $a$ if $N_b\supset N_a$ and $G_b\supset G_a$ (i.e., $b$ allocates a strict superset of  items and agents).
\begin{lemma}
Given two  allocations, $a$ and $b$, where  $b$ extends $a$
and $b$ is restricted optimal, then 
 $W(b)\geq W(a)$, 
with $W(b)>W(a)$ if $v_{ij}>0$ for all $i$, all $j$.  \label{lem:1} 
\end{lemma}
\begin{proof}
    Allocation $a$ is feasible in economy $E_a=(G_a,N_a)$ and thus feasible in economy $E_b=(G_b,N_b)$, 
 and so $W(b)\geq W(a)$ since $b$ is optimal on $E_b$.
 Moreover,  if items have strictly positive value then $W(b')>W(a)$ for  allocation $b'$, feasible in $E_b$, 
that extends $a$ through an arbitrary assignment of items $G_b\setminus G_a$ to $N_b\setminus N_a$. With this, we  have $W(b)\geq W(b')>W(a)$, and $b$ strictly improves welfare over $a$.
\end{proof}

\begin{proof} (of Theorem \ref{thm:welfare}.)
 Consider some deterministic allocation $a$ in the support of $A$,  
 and let $P_1=\Pr_A[a]$ denote the probability of assignment $a$.
Define $P_2=\sum_{b\in \mathrm{sup}(B), b \ \mathrm{extends}\ a} \Pr_B[b]$,
which is the  marginal probability of assignments that extend $a$.
We have

$$ \sum_{b \in \mathrm{sup}(B), b \ \mathrm{extends}\ a}\Pr\nolimits_B[b]=
     \sum_{b \in \mathrm{sup}(B), b \ \mathrm{extends}\ a}\prod_{i\in a}\Pr\nolimits_B[i]\cdot \prod_{i\in b, i\notin a} \Pr\nolimits_B[i]\\
    $$
     $$=\prod_{i\in a}\Pr\nolimits_B[i] \sum_{b \in \mathrm{sup}(B), b \ \mathrm{extends}\ a}
 \prod_{i\in b, i\notin a} \Pr\nolimits_B[i] = \prod_{i\in a}\Pr\nolimits_B[i] \cdot 1 \geq 
 \prod_{i\in a}\Pr\nolimits_A[i],$$

%
where the product structure is used to replace the marginalization over the part of the
assignment that extends $a$ by probability 1, and the inequality follows
since $B$ extends $A$. 
For any such $b$ that extends $a$, we have  $W(b)\geq W(a)$ by Lemma~\ref{lem:1}, 
where we use the property that $B$ is restricted optimal
and thus each $b$ in the support of $B$ is restricted optimal.
Then, since $P_2\geq P_1$, 
and considering all such $a$ in the support of $A$, 
we have  $W(B)\geq W(A)$. 
By considering 
the case of $v_{ij}>0$ for all $i$, all $j$, then $W(b)>W(a)$ 
by Lemma~\ref{lem:1}, and
we have $W(B)> W(A)$.
\end{proof}

By Theorem \ref{thm:welfare}, to argue that randomized allocation $B$ provides more welfare than randomized allocation $A$ it suffices to argue that (1) each assignment in the support of $B$ is 
restricted optimal, and (2) $B$ extends $A$ which means that $B$ allocates 
a superset of the 
items and each agent is allocated something with at least as much probability
in $B$ than $A$ (i.e., no agent faces more congestion).

The first set of conditions, namely Conditions 1, 2, and 3 in the main text, follow from reasoning about the following consistency
property, that needs to hold between perceived and true valuations.
\begin{definition}[Pointing consistency.]
An admissible allocation $\tilde{a}$ satisfies  {\em pointing consistency} if,
for every agent $i\in \tilde{N}$, the allocated item $\tilde{a}_i$  is 
 the best response 
 of $i$ at true valuations $v_{i}$.
\end{definition}

In other words, agent $i$ continues to prefer item $\tilde{a}_i$ 
at prices $p$ when moving from perceived valuation $\tilde{v}_i$ to true valuation $v_i$.
The following is immediate.
\begin{lemma} \label{lemma:pointingconsistency}
Admissible allocation $\tilde{a}$ is restricted optimal if the  pointing consistency
 condition holds. 
\end{lemma}
\begin{proof}
    $(\tilde{a},p)$ is a CE (defined with respect to true valuations) in  economy $(\tilde{G},\tilde{N})$. 
\end{proof}

\begin{lemma}
  Admissible allocation $\tilde{a}$ with margin $\Delta$
 satisfies pointing consistency (and therefore by Lemma \ref{lemma:pointingconsistency} is restricted optimal), 
when $v^H_{i\tilde{a}_i}\geq v^H_{ij}-\Delta$, for all $j\in \tilde{G}$, all $i\in \tilde{N}$.
\label{lem:2} 
\end{lemma}
\begin{proof}
For agent $i$, and any $j\neq \tilde{a}_i$, 
we have $v_{i\tilde{a}_i}-p_{\tilde{a}_i} = \tilde{v}_{i\tilde{a}_i}-p_{\tilde{a}_i}+v_{i\tilde{a}_i}^H\geq \tilde{v}_{ij}-p_j+\Delta + v^H_{ij}-\Delta=v_{ij}-p_j$,  and pointing consistency, 
where we substitute $\tilde{v}_{i\tilde{a}_i}-p_{\tilde{a}_i}\geq \tilde{v}_{ij}-p_j+\Delta$ 
(margin condition) and $v^H_{i\tilde{a}_i}\geq v^H_{ij}-\Delta$ (indifference assumption). 
\end{proof}

Considering a matrix with agents as rows and items as columns,
the
 property in Lemma~\ref{lem:2} is one of
 ``row-dominance"
for $\Delta=0$, 
 such that the value of
 an agent for its allocated item is weakly larger than that of
 every other item.  For this property, it suffices
that there
is  little variation in the hidden value for 
any items, which is in turn provided by 
the set of five conditions.
%

\begin{proof} (of Condition 1)
This  condition is sufficient for the hidden-value similarity of Lemma~\ref{lem:2},
 since
$v^H_{ij}-v^H_{i\tilde{a}_i}=\beta^\top_i(1-\mu)\odot(x_j-x_{\tilde{a}_i}) = 
\sum_{k: \mu_k=0} \beta_{ik}(x_{jk}-x_{\tilde{a}_ik})
\leq \sum_{k: \mu_k=0} |\beta_{ik}(x_{jk}-x_{\tilde{a}_ik})|
\leq \sum_{k: \mu_k=0} |x_{jk}-x_{\tilde{a}_ik}|
= |(1-\mu)\odot (x_j-x_{\tilde{a}_i})|_1\leq \Delta$,
where the penultimate inequality follows
from $0\leq \beta_{ik}\leq 1$.
\end{proof}


%
\begin{proof} (of Condition 2)
This condition is sufficient for the hidden-value similarity of Lemma~\ref{lem:2} since
$v^H_{ij}-v^H_{i\tilde{a}_i}=\beta^\top_i(1-\mu)\odot(x_j-x_{\tilde{a}_i}) =
\sum_{k: \mu_k=0} \beta_{ik}(x_{jk}-x_{\tilde{a}_ik})
\leq \sum_{k: \mu_k=0} |\beta_{ik}(x_{jk}-x_{\tilde{a}_ik})|
\leq \sum_{k: \mu_k=0} |\beta_{ik}|= |(1-\mu)\odot \beta_i|_1\leq \Delta$,
where the penultimate inequality follows
from $0\leq x_{j'k}\leq 1$,  for all item $j'$ and features $k$.
\end{proof}
%
%




  
\begin{proof} (of Condition 3) 
By the margin property, we have $\tilde{v}_{i\tilde{a}_i}-p_{\tilde{a}_i}\geq \tilde{v}_{ij}-p_j+\Delta$, for any $j\in \tilde{G}$, 
and adding $p_{\tilde{a}_i}\geq p_j-\Delta$ (price variation) we have $\tilde{v}_{i\tilde{z}_i}\geq \tilde{v}_{ij}$, 
and so $\tilde{a}_i$ is the top item for $i$ given revealed features. Given this, we have
$v^H_{i\tilde{a}_i}+\Delta\geq v^H_{ij}$, for all $j\in \tilde{G}$ (top-item value
consistency), which is the 
hidden-value similarity condition of Lemma~\ref{lem:2}.
\end{proof}

The second set of conditions, namely Conditions 4 and 5 in the main text, come from considering an 
{\em approximate column dominance property} on hidden valuations. 
Considering a matrix with agents as rows and items as columns, 
column dominance means that 
the agent to which an item is allocated has 
weakly larger value for the item than that of any other 
agent.
%
%
\begin{definition}[Approximate column dominance]
An admissible allocation $\tilde{a}$ with margin $\Delta$
satisfies {\em approximate column dominance} if,
for each item $j\in \tilde{G}$
 and agent $i$ allocated item $j$, 
we have  $v^H_{ij}\geq v^H_{i'j}-\Delta$, for all 
$i'\in \tilde{N}$.
\end{definition}

\begin{lemma}
Admissible allocation $\tilde{a}$ with margin $\Delta$ is restricted optimal if the 
approximate column dominance condition holds. 
\end{lemma}
\begin{proof}
  First, given margin $\Delta$ then 
  \begin{align}
  \sum_{i\in \tilde{N}}\sum_{j\in \tilde{G}}\tilde{a}_{ij}\tilde{v}_{ij}
  \geq \sum_{i\in \tilde{N}}\sum_{j\in \tilde{G}}a'_{ij}\tilde{v}_{ij}+|\tilde{G}|\Delta, \quad \mbox{all $a'$,}\label{eq:23}
  \end{align}
since we can reduce $\tilde{v}_{i\tilde{a}_i}$ by $\Delta$ to each agent $i$, leaving
the rest of the perceived values unchanged, and
this item will still be in the demand set of the agent, and thus $(\tilde{a},p)$
would be a CE for these adjusted, perceived values (perceived, not true values).
Thus, the total perceived value
for $\tilde{a}$ is  at least $|\tilde{G}| \cdot \Delta$ better than the
total perceived value 
of the next best
  allocation, considering economy $(\tilde{G},\tilde{N})$.

Second, we argue that approximate column dominance implies that $\tilde{a}$ approximately optimizes
the total hidden value. 
First, suppose we have exact column dominance, 
with   $v^H_{ij}\geq v^H_{i'j}$, for all $i'\in \tilde{N}$, 
 item $j\in \tilde{G}$, 
 and agent $i$ allocated item $j$.
Then, allocation $\tilde{a}$ 
would maximize hidden values.
To see this, consider  the transpose of this assignment problem, so that
agents become items and items become agents. 
This maintains the optimal assignment. 
$\tilde{a}$ is optimal in the transpose economy
 by considering  zero price on each 
agent and items bidding on agents: 
 by column dominance, each agent is 
allocated its most preferred agent.
By  approximate column dominance, we have
 \begin{align}
  \sum_{i\in \tilde{N}}\sum_{j\in \tilde{G}}\tilde{a}_{ij} v^H_{ij}+|\tilde{G}|\Delta &\geq 
  \sum_{i\in \tilde{N}}\sum_{j\in \tilde{G}}a'_{ij}v^H_{ij}, \quad \mbox{all $a'$,}\label{eq:24}
  \end{align}
and $\tilde{a}$ approximately optimizes total hidden value. 
This  follows by  considering the transpose economy, and noting that 
if we increase $v^H_{i\tilde{a}_i}$ by $\Delta$, to each agent $i$, 
leaving the other  hidden values
unchanged, we have exact column dominance and optimality of $\tilde{a}$.
 This means that $\tilde{a}$ is at most $|\tilde{G}| \cdot \Delta$ worse
than any other allocation. 
Combining~\eqref{eq:23}  for perceived values and~\eqref{eq:24} for hidden values, we have
\begin{align}
 \sum_{i\in \tilde{N}}\sum_{j\in \tilde{G}}\tilde{a}_{ij} (\tilde{v}_{ij}+v^H_{ij})+|\tilde{G}|\Delta 
 \geq  \sum_{i\in \tilde{N}}\sum_{j\in \tilde{G}}a'_{ij} (\tilde{v}_{ij}+v^H_{ij})+|\tilde{G}|\Delta, \quad \mbox{all $a'$,}
\end{align}
and thus $\tilde{a}$ is restricted optimal, since $\tilde{v}_ij+v^H_{ij}=v_{ij}$.
\end{proof}

It suffices for approximate column dominance 
that there is  little variation  across agents in their 
hidden value 
for an item,   which is in turn provided by 
the following properties (approximate column dominance is also achieved
by Condition 2).
 %


\begin{proof} (of Condition 4)
When this condition holds, we have
$|v^H_{ij}-v^H_{i'j}|=|\beta^\top_i(1-\mu)\odot x_j - \beta^\top_{i'}(1-\mu)\odot x_j|=
|(\beta_i-\beta_{i'})^\top (1-\mu)\odot x_j| =
\sum_{k:\mu_k=0} |(\beta_{ik}-\beta_{i'k})x_{jk}|\leq \sum_{k:\mu_k=0}
|x_{jk}|= |(1-\mu)\odot x_j|_1\leq \Delta$,
where the penultimate inequality follows from $0\leq \beta_{i''k}\leq 1$, for any $i''$,
any $k$. This establishes that all pairs of agents have similar
hidden value
for any given item, and in particular
approximate column dominance and $v^H_{i'j}-v^H_{ij}\leq \Delta$
for agent $i$ allocated item $j$ in $\tilde{a}$ and any
other agent $i'\in \tilde{N}$.
\end{proof}



\begin{proof} (of Condition 5)
With this,  we have 
$|v^H_{ij}-v^H_{i'j}|=|\beta^\top_i(1-\mu)\odot x_j - \beta^\top_{i'}(1-\mu)\odot x_j|=
|(\beta_i-\beta_{i'})^\top (1-\mu)\odot x_j|=
|\sum_{k:\mu_k=0}(\beta_{ik}-\beta_{i'k}) x_{jk}|\leq \sum_{k:\mu_k=0}|(\beta_{ik}-\beta_{i'k}) x_{jk}|
\leq \sum_{k:\mu_k=0}|\beta_{ik}-\beta_{i'k}|= |(1-\mu)\odot(\beta_i-\beta_{i'})|_1
\leq \Delta$, where the penultimate inequality 
follows from $0\leq x_{jk}\leq 1$ for any item $j$.
This establishes that all pairs of agents have similar
hidden value
for any given item, and in particular
approximate column dominance and $v^H_{i'j}-v^H_{ij}\leq \Delta$
for agent $i$ allocated item $j$ in $\tilde{a}$ and any
other agent $i'\in \tilde{N}$.
\end{proof}

\section{Method: Additional Details} \label{apx:method}

Although our approach makes use of prediction, in essence, the problem of finding optimal representations is counterfactual in nature.
This is because choosing a good mask requires anticipating what users \emph{would have chosen} had they made choices under this new mask; these may differ from the choices made in the observed data.
As such, decongestion by representation is a policy problem.
This has two implications: on how data is collected, and on how to predict well.


\subsection{Default policy} \label{apx:default_policy}
To facilitate learning,
we assume that training data is collected under representations determined according to a `default' stochastic masking policy, $\pi_0$.
The degree to which we can expect data to be useful for learning counterfactual masks depends on how informative $\pi_0$ of other representations.
In particular,
if there is sufficient variation in masks generated by $\pi_0$,
then in principle it should be possible to generalize well from $\pi_0$ to a learned masking policy, $\pihat$ (which can be deterministic).
We imagine $\pi_0$ as concentrated around some reasonable default choice of mask, e.g., as elicited from a predictive model, or which includes features estimated to be most informative of user values.
However, $\pi_0$ must include some degree of randomization;
in particular, to enable learning, we require $\pi_0$ to have full support over all masks, i.e., have $P_{\pi_0}(\mask) \ge \epsilon$ for all $\mask$ and for some $\epsilon > 0$.
In our experiments we set $\pi_0$ to have most probability mass concentrated around features coming from a predictive baseline (e.g., \pricepred),
but with some probability mass assigned to other features.

\subsection{Counterfactual prediction} \label{sec:counterfactual}

Representation learning is counterfactual since choices at test time depend on the learned mask. At train time, counterfactuality manifests in predictions: for any given $\mask$ examined during training, our objective must emulate choices $y(\mask)$, which rely on $v$,
via predictions $\yhat(\mask)$, which rely only on observed features $u,X$ and prices $p$.
As such, we must make use of choice data sampled from $\pi_0$ to predict choices to be made under differing $\mask$.
There is extensive literature on learning under distribution shift,
and in principle any method for off-policy learning should be applicable to our case.
One prominent approach relies on \emph{inverse propensity weights}, which weight examples in the predictive learning objective  according to the ratio of train- to test-probabilities, 
\[
w_\pi(\mask) = \frac{\prob{\pi}{\mask}}{\prob{\pi_0}{\mask}}
\]
for all masks $\mask$ in the training data,
which are then used to modify Eq.~\eqref{eq:prediction_objective} into:
\begin{equation}
\label{eq:prediction_objective_weighted}
\fhat = \argmin_{f \in F} \sum\nolimits_{(M,y) \in \smplst} \sum\nolimits_{i \in [n]}
w_\pi(\mask)
\Loss(y_i, f(X,p;u_i,\mask))
\end{equation}
For the default policy,
propensities $\rho=\prob{\pi_0}{\mask}$ are assumed to be collected and accessible as part of the training set.
For the current policy $\pi$, $\prob{\pi}{\mask}$ can be approximated 
from the Multivariate Wallenius' Noncentral Hypergeometric Distribution,
which describes the distribution of sampling without replacement from a non-uniform multinomial distribution.
\extended{(e.g., \citep{fog2008sampling})}.
This makes the predictive objective unbiased with respect to the shifted target distribution,
and as a result, makes Eq.~\eqref{eq:proxy_differentiable} appropriate for the current $\pi$.

In our case, because the shifted distributions are not set a-priori, but rather, are determined by the learned representations themselves, our problem is in fact one of \emph{decision-dependent distribution shift}. Our proposed solution to this is to alternate between: (i) optimizing $f_t$ in Eq.~\eqref{eq:prediction_objective} to predict well for data corresponding to the current mask $\mask_{t-1}$, holding parameters $\theta_{t-1}$ fixed;
and (ii) optimizing $\mask_t$ by updating parameters $\theta_t$ in Eq.~\eqref{eq:proxy_differentiable} for a fixed $f_t$. That is, we alternate between training the predictor on a fixed choice distribution, and optimizing representations for a fixed choice predictor.

Nonetheless, in our experiments we have found that simply training $f$ to predict well on $\pi_0$---without any reweighing or adjustments---obtained good overall performance, despite an observed reduction in the predictive performance of $f$ on counterfactual choices made under the learned $\mask$ (relative to predictive performance on $\pi_0$).



\section{Experimental Details: Synthetic}

Experiments were implemented in Python.
See supplementary material for code. \extended{reffer to github url}

\paragraph{Prices.} \label{sec:app_prices}
For computing CE prices we used the cvxpy convex optimization package to implement Eq.~\eqref{eq:ce_dual}. This give \emph{some} price vector in the core. To interpolate between buyer-optimal and seller-optimal core prices, we adjust Eq.~\eqref{eq:ce_dual} by:
(i) solving the original Eq.~\eqref{eq:ce_dual} to obtain the optimal dual objective value;
(ii) adding a constraint for the objective value to obtain the optimal value; 
and (iii) modifying the current objective to either minimize prices (for buyer-optimal) or maximize prices (for seller-optimal).

\paragraph{Preferences.}
To generate mixture value matrices, we first sample two random item features matrices $X_{(1)},X_{(2)} \in [0,1]^{n \times d}$ with entries sampled independently from the uniform distribution over $[0,1]$.
Next, we generate a fully-heterogeneous value matrix $V_{(1)}=V_{\text{het}}$, and a fully-homogeneous matrix $V_{(2)} = V_{\text{hom}}$.
The heterogeneous matrix is constructed by taking the preference vector
$(m,m-1,\dots,1)$, normalized to $[0,1]$,
and creating a circulant matrix, so that user $i$ most prefers item $i$, and then preferences decreasing in items with increasing indices (modulo $m$).
The homogeneous matrix is constructed by assigning the same preference vector to all users.\footnote{We also experimented with adding noise to each $V_{(i)}$ (small enough to retain preferences), but did not observe this to have any significant impact on results.}
Finally, to obtain the corresponding $B_{(i)}$, we solve for the convex objective
$\min_{B \ge 0} \|B X^\top - V_{(i)}\|_2$,
and 
for $\alpha \in [0,1]$,
set $B_\alpha = (1-\alpha) B_{\text{het}} + \alpha B_{\text{hom}}$ and
$X=X_{(1)}+X_{(2)}$,
which gives the desired $V_\alpha = B_\alpha X^\top$.

\paragraph{Optimization.}
Because we consider small $n,m,d$, and because as designers of the experiment we have access to $v$, in this experiment we are able to compute measures that rely on $v$.
In particular, by enumerating over all $d$-choose-$k$ possible masks,
we are able to exactly optimize the considered objectives,
compute the welfare oracle upper bound,
and obtain all optimal solutions in case of ties (as in the case of the decongestion objective).

\section{Experimental Details: Real Data} \label{sec:appendix_real_data_detailes}


\subsection{Data generation}

\paragraph{Data and preprocessing.}
The Movielens 100k dataset is available at 
\url{https://grouplens.org/datasets/movielens/100k/}.
NMF on partial rating matrix was done by \emph{surprise}\footnote{\url{https://surpriselib.com/}} python package 
\extended{cite surpriselib paper (see website)}
For Movielens, as rating vlues range from 1 to 5, we normalize then into $[0,1]$ by dividing the user preferences matrix $B$ by a factor of 5.

\paragraph{Prices.}
CE prices $p^*$ were computed by solving the dual LP in Eq.~\eqref{eq:ce_dual},
similarly to the synthetic experiments.
For varying prices between buyer-optimal ($p_{\mathrm{buyer}}$) 
and seller-optimal ($p_{\mathrm{seller}}$) CE prices,
we interpolate between
$p_{\mathrm{buyer}}$ and $p^*$ for $\gamma \in [0,0.5]$,
and between $p^*$ and 
$p_{\mathrm{seller}}$ for $\gamma \in (0.5,1]$,
this since interpolating directly between $p_{\mathrm{buyer}}$ and $p_{\mathrm{seller}}$ is prone to exhibiting many within-user ties as an artifact,
and since $p^*$ is often very close to the average price point
$(p_{\mathrm{buyer}}+p_{\mathrm{seller}})/2$.


\paragraph{Default masking policy.}
As discussed in Appendix~\ref{apx:default_policy}, our method requires training data to be based on masks generated from a default masking policy, $\pi_0$.
We defined $\pi_0$ to be concentrated around the features selected by the \pricepred\ predictive baseline, but ensure all features have strictly positive probability.
In particular, let $\mask_0$ be the mask including the set of $k$ features as chosen by \pricepred.
Then $\theta$ for $\pi_0$ is constructed as follows:
first, we assign $\theta_i = 1$ for all $i \notin \mask_0$; 
then, we assign $\theta_i=3$ for all $i \in \mask_0$;
finally, we normalize $\theta$ using a softmax with temperature 0.05, 
this resulting in a distribution over features that strictly positive everywhere but at the same time tightly concentrated around $\mask_0$, 
and in a way which depends on $k$ (since different $k$ lead to different normalizations).
An example $\pi_0$ is shown in Fig~\ref{fig:choce_pred+pi0} (left).


\begin{figure}[t!]
\centering
\includegraphics[width=0.8\linewidth]{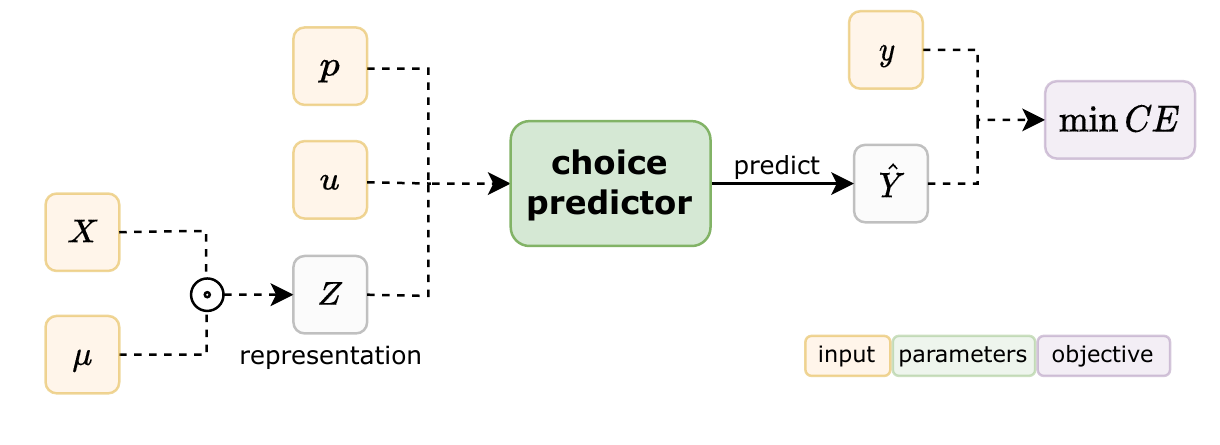}
\caption{A schematic illustration of our choice prediction model.}
\label{fig: choice prediction}
\end{figure}

\subsection{Our framework}

\paragraph{Choice prediction.}
The choice prediction model $f$ is trained to predict choices (including null choices) from training data.
For the class of predictors $F=\{f\}$,
we use item-wise score-based bilinear classifiers parameterized by $W \in \R^{d \times d'}$, 
namely:
\[
f_W(X,p;u,\mask) = \argmax_{x \in X} u^\top W (x\odot\mask) - p
\]
There are implemented as a single dense linear layer,
and for training, the argmax is replaced with a differentiable softmax.
We found learning to be well-behaved even under low softmax temperatures, and hence use $\tau_f = 5\mathrm{e}{-4}$ throughout.
For training we used cross entropy loss as the objective.
For optimization we Adam for 150 epochs, with learning rate of $1\mathrm{e}{-3}$ and batch size of $20$.
See Figure~\ref{fig: choice prediction} for a schematic illustration.
Training data used to train $f$ includes user choices $y$ made on the basis masks $\mask$ sampled from the default policy, $\mask \sim \pi_0$. Nonetheless, as described in \ref{sec:counterfactual},
recall that we would like $f$ to predict well on the final learned mask $\mask$, but also on other masks encountered during training, and more broadly---on any possible mask.
Figure \ref{fig:choce_pred+pi0} (center+right) shows, for $d=12$ and $d=100$ and as a function of $k$, the accuracy of $f$ on (i) data representative of the training distribution (i.e., masks sampled from $\pi_0$), and (ii) data which includes masks sampled uniformly at random from the set of all possible $k$-sized masks.
As can be seen, across all $k$, performance on arbitrary masks 
closely matches in-distribution performance for $d=12$,
and remains relatively high for $d=100$
(vs. random performance at 5\% for $m=20$).

\paragraph{Representation learning.}
The full-framework model consists of a Gumbel-top-$k$ layer,
applied on top of a `frozen' choice prediction model $f$,
pre-trained as described above.
The Gumbel-top-$k$ layer has $d$ trainable parameters $\theta \in \Theta = \R^d$;
once passed through an additional softmax layer, this constitutes a distribution over features.
As described in the main paper, given this distribution, we generate random masks by independently sampling $k$ features $i \sim \theta_i$ without replacement (and re-normalizing $\theta$).
However, to ensure our framework is differentiable,
we use a \emph{relaxed-top-$k$} procedure for generating `soft' $k$-sized masks,
and for each batch, we sample in this way $N$ soft masks, for which we adopt the procedure of \citep{xie2019subsets}.

Given a sampled batch of masks $\{\mask\}$, these are then plugged in to the prediction model $f$ to obtain $\yhat(\mask)$, and finally our proxy-loss $-\tilde{W}$ is computed.
Optimization was carried out using the Adam optimizer for 300 epochs
(at which learning converged for most cases) and
with a learning rate of $1\mathrm{e}{-2}$.
We set $N=20$, and use temperatures
$\tau_{\text{Gumbel}}=2$ for the Gumbel softmax, 
$\tau_{\text{top-$k$}}=0.2$ for the relaxed top-$k$,
and $\tau_f=0.01$ for the softmax in the pre-trained predictive model $f$.
Since the selection of the top-$k$ features admits several relaxations,
for larger $k>d/2$, we have found it useful to instead consider $k \gets d-k$ in learning, and then correspondingly use `inverted' masks $\mask \gets 1-\mask$. 

\begin{figure}
\centering
\includegraphics[width=0.32\linewidth]{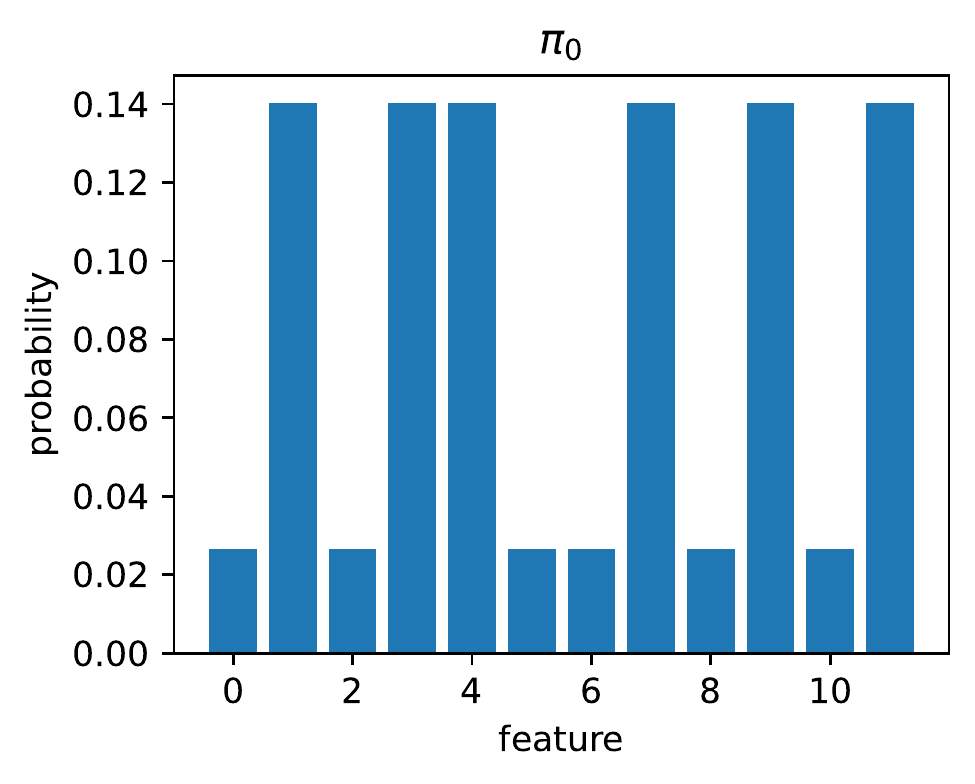}
\includegraphics[width=0.32\linewidth]{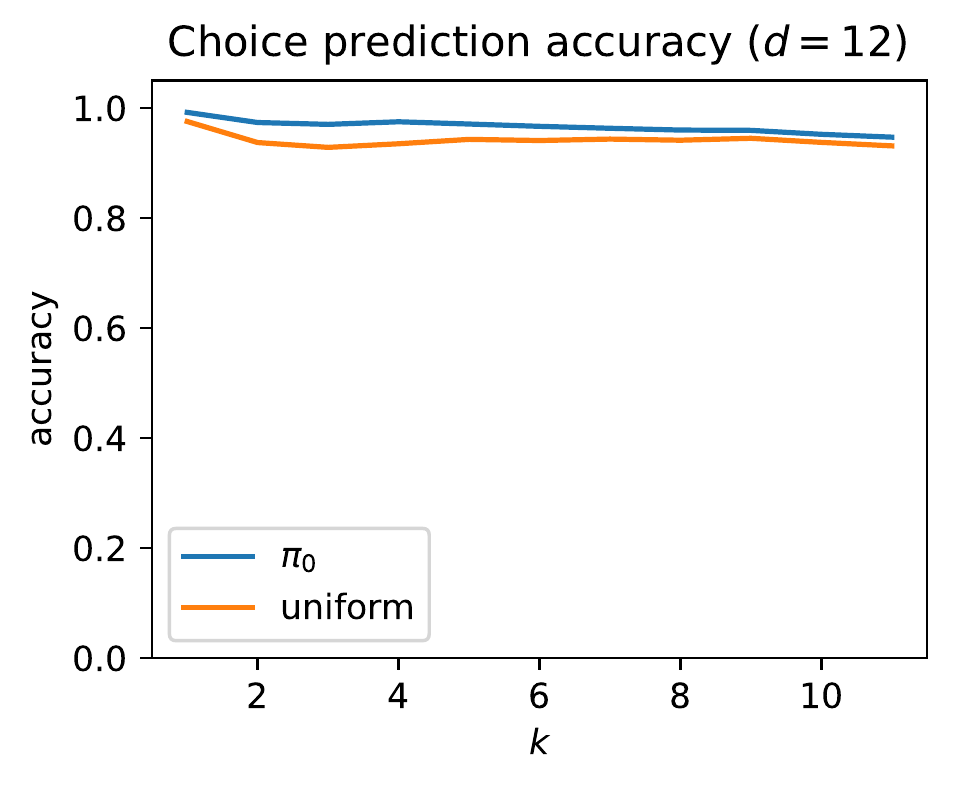}
\includegraphics[width=0.32\linewidth]{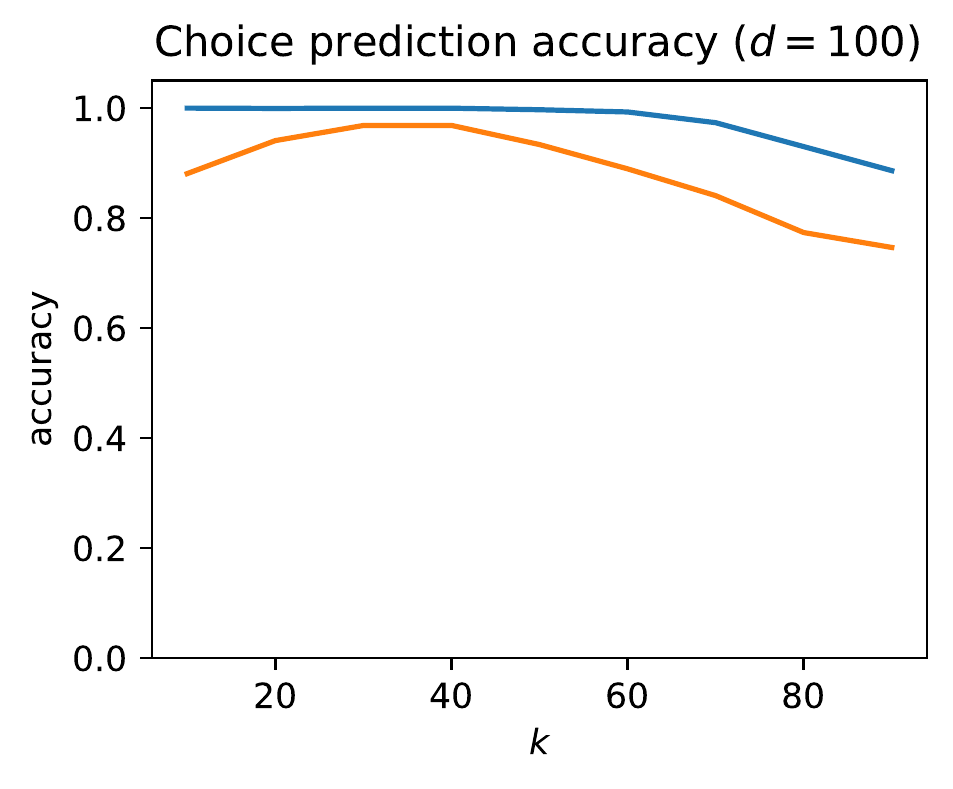}
\caption{
\textbf{(Left):}
An example for default policy $\pi_0$ ($d=12, k=6$). Each feature (in x-axis) is assigned with a probability (y-axis) to be drawn from the categorical distribution. High probability is assigned to predictive features from \pricepred.
\textbf{(Center+right):}
Accuracy of our choice predictor $f$, for Movielens with $d=12$ (center) and $d=100$ (right).
}
\label{fig:choce_pred+pi0}
\end{figure}

\paragraph{Variants.}
As noted, we evaluate three variants of our approach that differ in their usage at test-time:
\begin{itemize}
\item \DbRtopk:
Constructs a mask from the top-$k$ entries in the learned $\thetahat$.

\item \DbRmask:
A heuristic for choosing a mask on the basis of training data.
Here we sample 20 masks $\hat{\mask}$ according to the multinomial distribution defined by the learned $\thetahat$, and commit to the sampled mask obtaining the lowest value on the proxy objective.

\item \DbRpolicy:
Emulates using $\thetahat$ as a masking policy $\pihat=\pi_\thetahat$.
Here we sample 50 masks $\mask \sim \pihat$,
evaluate for each sampled mask its performance on the entire test set,
and average.

\end{itemize}

\subsection{Baselines}

\begin{itemize}
\item Price predictive (\pricepred):
Selects the $k$ most informative features for the regression task of predicting the price of items, based on item its features.
Data includes features and prices for all items that appear in the dataset (recall markets include the same set of items, but items can be priced differently per market).
We use the Lasso path (implementation by \emph{scikit-learn}\footnote{\url{https://scikit-learn.org/stable/modules/generated/sklearn.linear_model.lasso_path}}) to order features in terms of their importance for prediction,
and take as a mask the top $k$ features in that order.

\item Choice predictive (\choicepred):
Selects the $k$ most informative features for the classification task of predicting user choices from user and item features.
For this baseline we use the predictive model $f$,
where we interpret learned weights $W=\hat{T}$ as an estimated of the true underlying mapping $T$ between user features $u$ and (unobserved) preferences $\beta$.
Inferred parameters $\hat{T}$ are then used to obtain
estimated preferences per user via $\hat{\beta} = u\hat{T}$.
We then average preferences over users, to obtain preferences representative of an `average' user,
and from which we take the top $k$-features,
we we interpret as accounting for the largest proportion of value.

\item Random (\random): Here we report performance averaged over 100 random  masks sampled uniformly from the set of all $k$-sized masks.
\end{itemize}

\subsection{Implementation}
\paragraph{Code.}
All code is written in python.
All methods and baselines are implemented and trained with Tensorflow\footnote{\url{https://www.tensorflow.org/}} 2.11  and using Keras.
CE prices were computed using the convex programming package cvxpy\footnote{\url{https://www.cvxpy.org/}}.

\paragraph{Hardware.}
All experiments were run on a Linux machine wih AMD EPYC 7713 64-Core processors. For speedup runs were parallelized each across 4 CPUs.

\paragraph{Runtime.}
Runtime for a single experimental instance of the entire pipeline 
was clocked as:
\begin{itemize}
    \item $\approx6.5$ minutes for the $d=12$ setting
    \item $\approx13.5$ minutes for the $d=100$ setting
\end{itemize}
Data creation was employed once at the onset.

\section{Additional experimental results: synthetic data}  \label{apx:synth-additional}

\begin{figure}[t!]
    \centering
    \includegraphics[width=\linewidth]{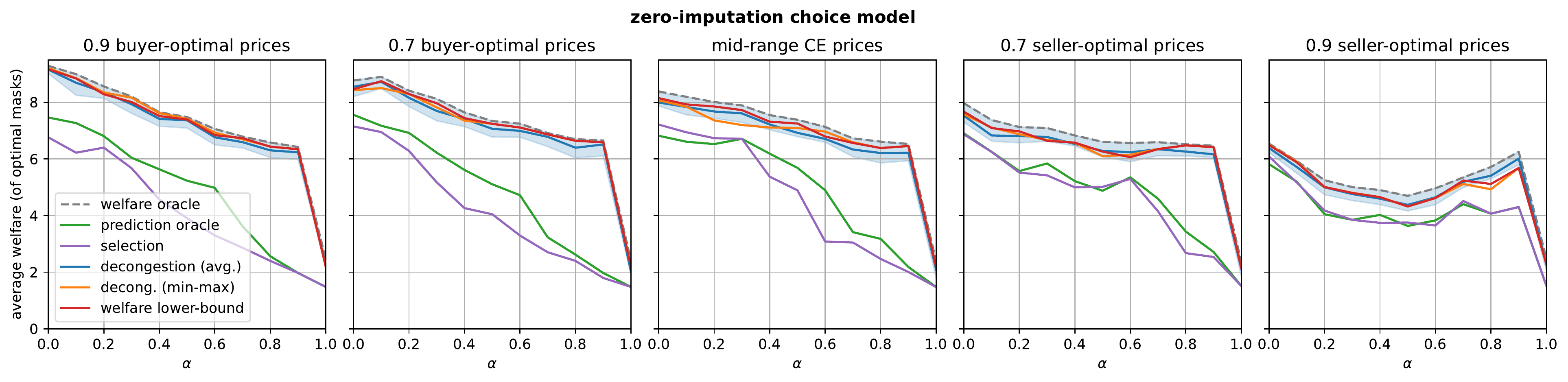} 
    \\
    \includegraphics[width=\linewidth]{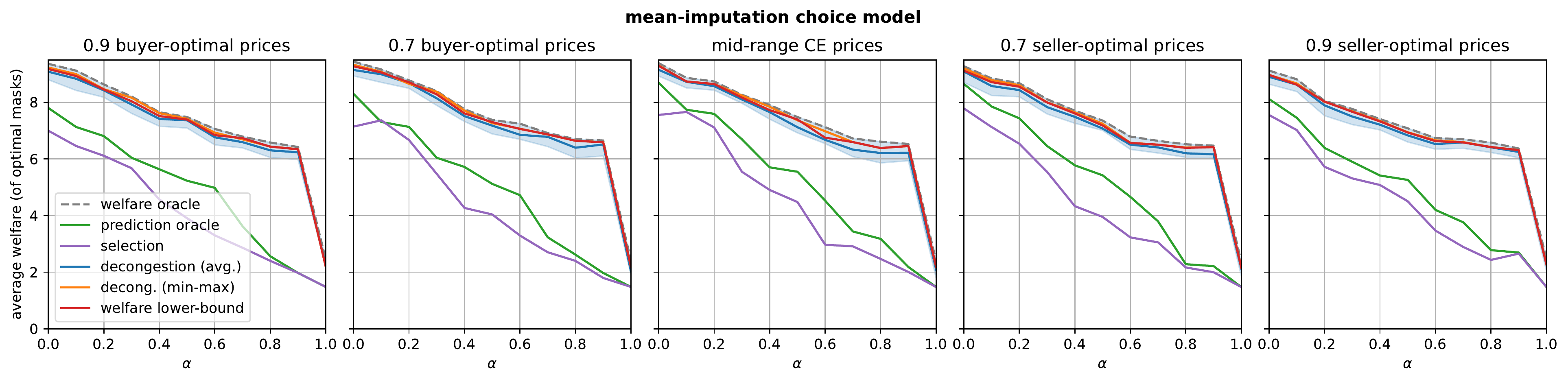} 
    \caption{
    A replication of results on synthetic data (Sec.~\ref{sec:exp_synth}) on an alternative choice model based on mean-imputed values (top), in comparison to the zero-imputed choice model studied in the main paper (bottom).
    Evaluation is also extended to additional CE price schemes.
    }
    \label{fig:zero+mean_impute}
\end{figure}

\subsection{Mean-imputation choice model} \label{apx:synth-additional}
In this section we replicate our main synthetic experiment in Sec.~\ref{sec:exp_synth} on a different user choice model.
In the main part of the paper, we model users as contending with the partial information depicted in representations by assuming that unobserved features do not contribute towards the item's value.

In particular, here we consider users who replace masked features with \emph{mean-imputed values}:
for example, if some feature $\ell$ is masked,
then features $x_{j \ell}$ are replaced with the `average' feature,
$\bar{x}_\ell = \frac{1}{m} \sum_{j'} x_{j' \ell}$,
computed over and assigned to all market items $j$.
This is in contrast to the choice model defined in (see Sec.~\ref{sec:setup}) which relies on zero-imputed values.
The main difference is that with mean imputation,
(i) perceived values can also be \emph{higher} than true values (e.g., if $\bar{x}_\ell > x_{j \ell}$ for some item $j$);
and (ii) our proxy welfare objective in Eq.~\eqref{eq:welfare_lower_bound} is no longer a lower bound on true welfare.
Nonetheless, we conjecture that if mean-imputed perceived values do not dramatically distort inherent true values, then proxy welfare can still be expected to perform well as an approximation of true welfare.
\looseness=-1

Figure~\ref{fig:zero+mean_impute} shows performance for all methods considered in Sec.~\ref{sec:exp_synth} on mean-imputed choice behavior,
for increasing $k$ and for a range of possible CE prices.
For comparison we also include results for our main zero-imputed choice model (mid-range CE prices are used in Sec.~\ref{sec:exp_synth}).
As can be seen, our approach retains performance for mean-imputed choices across all considered pricing schemes. Whereas for zero-imputed choices overall welfare decreases when prices are higher (likely since higher prices increase null choices),
mean-imputed choices exhibit a similar degree of welfare regardless of the particular price range.



\begin{figure}[t!]
    \centering
    \includegraphics[width=0.35\linewidth]{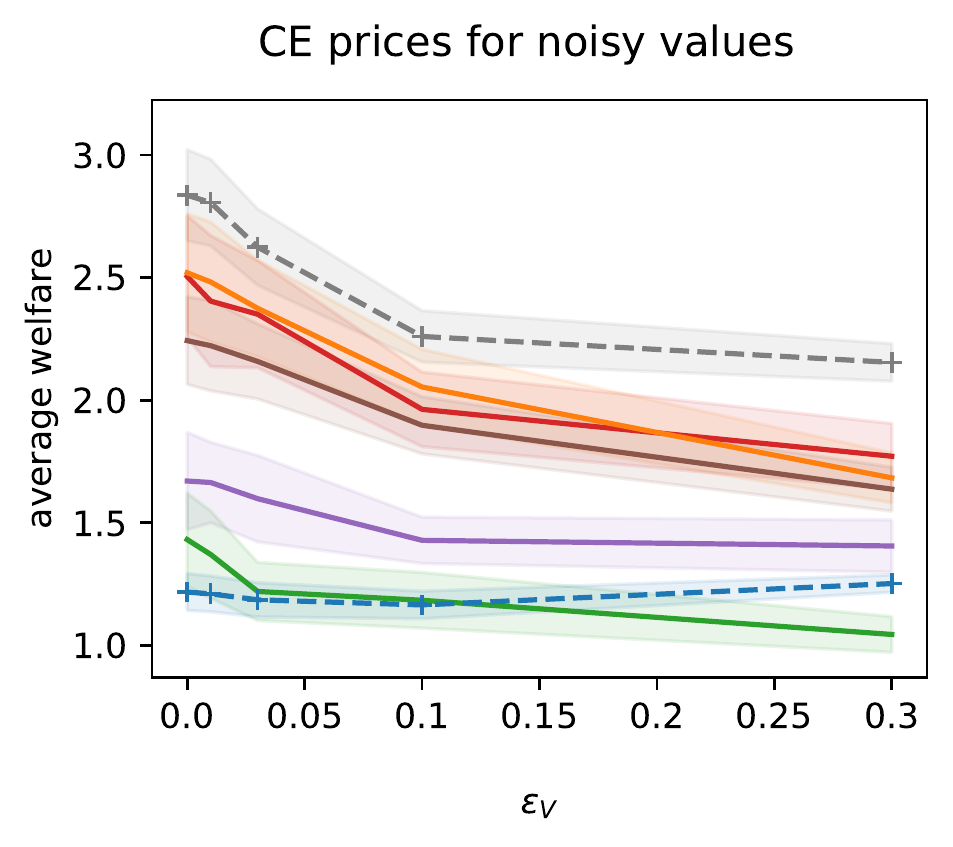} 
    \includegraphics[width=0.355\linewidth]{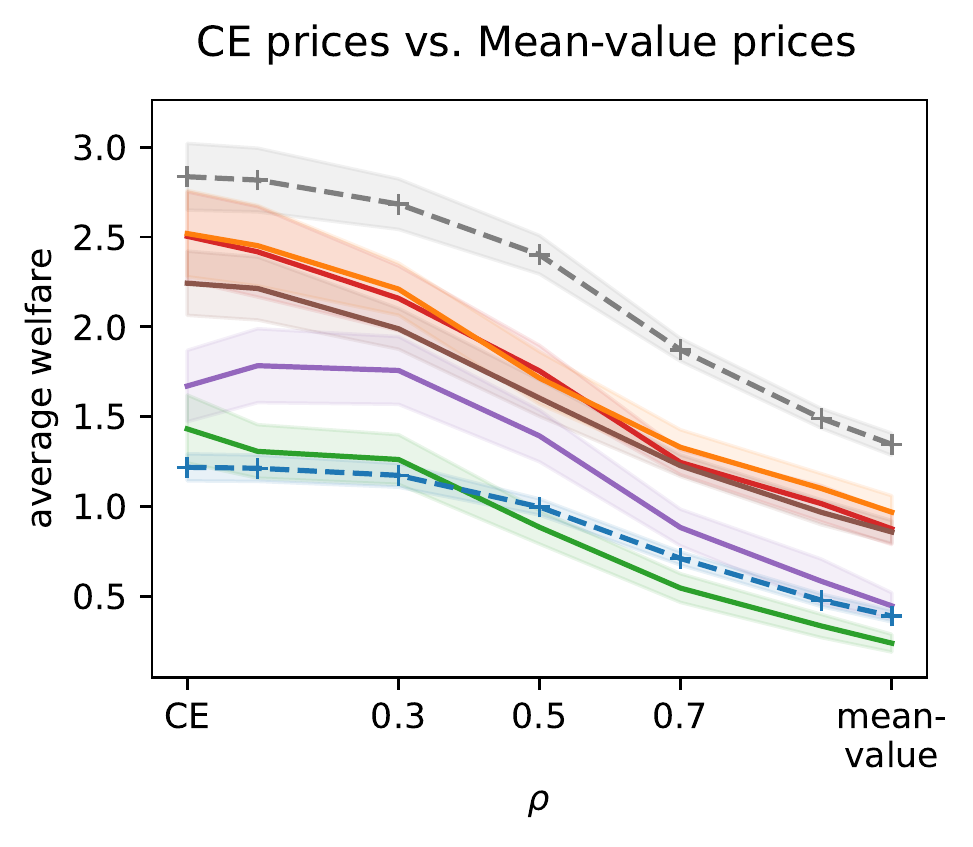}     \includegraphics[width=0.18\linewidth]{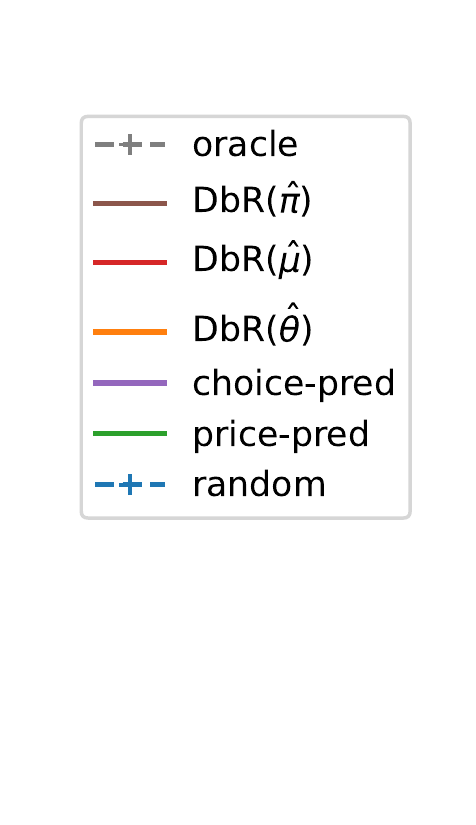}
    \caption{Other prices schemes for real data experiments.
    \textbf{(Left:)} Welfare (absolute) obtained for CE prices computed on noisy valuations $v+\epsilon_v$ for increasing additive noise $\epsilon_v$. 
    \textbf{(Right:)} Welfare (absolute) obtained for prices that interpolate between CE prices and heuristic (non-CE) prices set to average user values.} 
    \label{fig:additional_prices_plots}
\end{figure}

\begin{figure}[b!]
    \centering
    \includegraphics[width=0.38\linewidth]{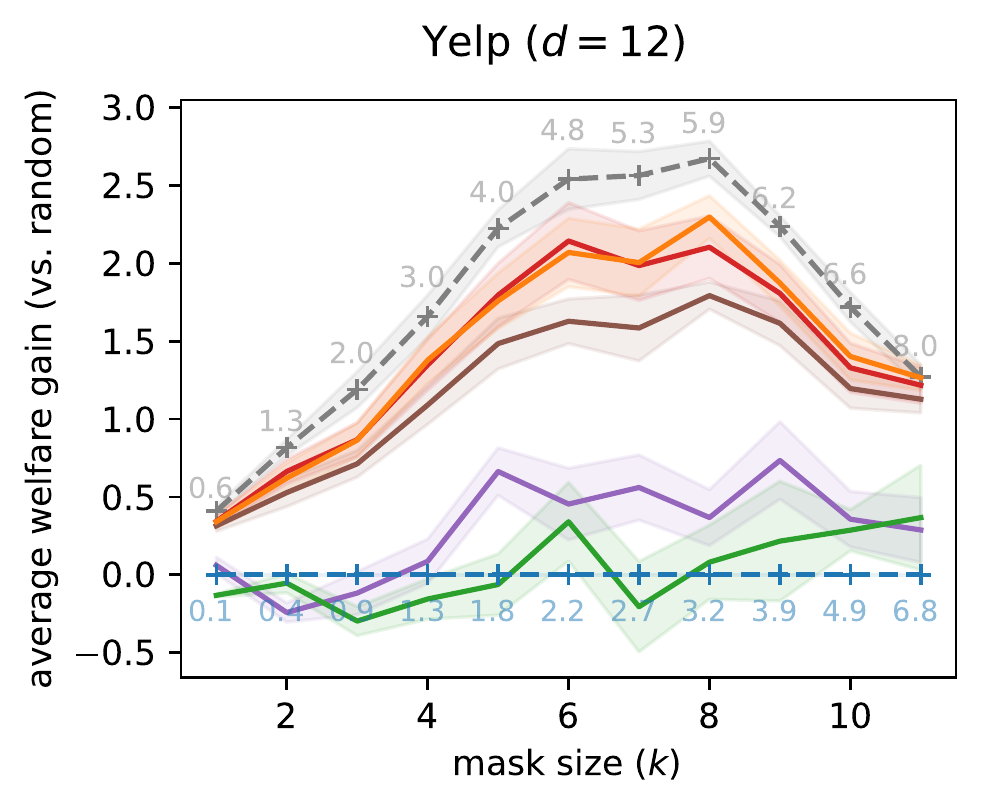} 
    \includegraphics[width=0.5\linewidth]{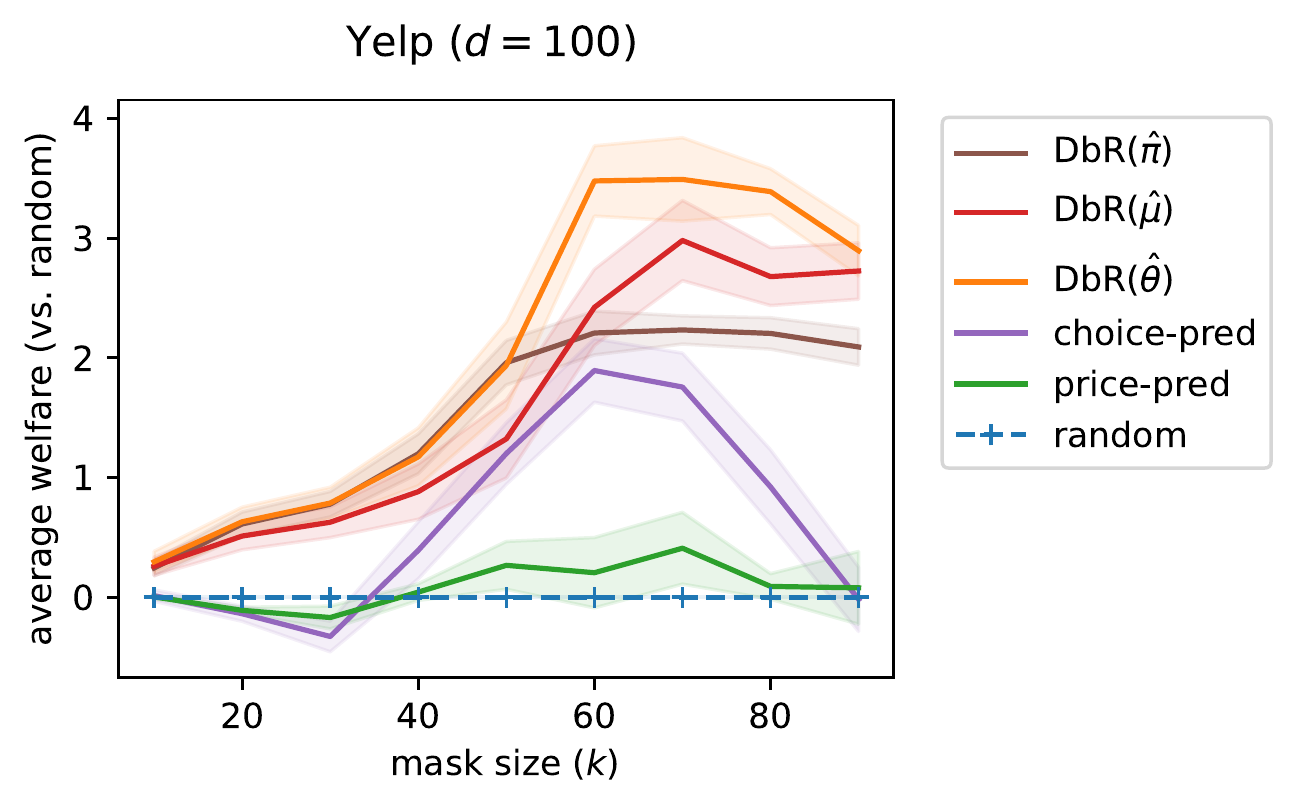} \caption{
    A replication of our main experiment on the Yelp restaurant reviews dataset.
    \looseness=-1
    }
    \label{fig:yelp}
\end{figure}

\section{Additional experimental results: real data}  \label{apx:real-additional}

\subsection{Additional dataset - Yelp restaurants} \label{apx:real-yelp}
Here we present a replication of our main experiment in 
and an additional dataset---the restaurants portion of the Yelp! reviews dataset, which is publicly-available.\footnote{\url{https://www.kaggle.com/datasets/yelp-dataset/yelp-dataset}}
We use the same preprocessing procedure and experimental setup as for MovieLens (see Appendix~\ref{sec:appendix_real_data_detailes}),
but also filter to keep restaurants that received at least 20 reviews, and users that gave at least 20 reviews.
Figure~\ref{fig:yelp} shows results. 
As can be seen,
general trends are qualitatively similar to the those observed for the Movielens dataset (Figure~\ref{fig:real} in the main paper).

\subsection{Pricing schemes} \label{apx:real-pricing_schemes}
Here we present robustness results for additional pricing schemes, which complement our results from Sec.~\ref{sec:prices_experiments}.
In particular, we examined performance for:
\begin{itemize}
\item 
Prices set by solving Eq.~\eqref{eq:ce_dual}, but for noisy valuations $v+\epsilon_v$, for increasing levels of noise $\epsilon$.
These simulate a setting where prices are CE, but for the `wrong' valuations, $p^*(v+\epsilon_v)$.
Results are shown in Figure~\ref{fig:additional_prices_plots} (left).

\item
Prices that interpolate from mid-range CE prices (as in the main experiments) to heuristically-set, non-CE prices. Specifically, here we use prices based on average values assigned by users to items. 
Results are shown in Figure~\ref{fig:additional_prices_plots} (right).
\end{itemize}

Overall, as in the main paper, moving away from CE prices causes a reduction in potential welfare, and in the performance of all methods.
Results here demonstrate that in the above additional pricing settings, our approach is still robust in that it maintains it's relative performance compared to baselines and the welfare oracle.



\subsection{The importance of $\lambda$} \label{sec:appendix_lambda}
In principle, and due to the counterfactual nature of learning representations (see Appendix~\ref{apx:method},
tuning $\lambda$ requires experimentation, i.e., deploying a learned masking model $\hat{\mask}$ trained on data using some $\lambda$,
to be evaluated on other candidate $\lambda'$.
Nonetheless, in our experiments we observe that learning is fairly robust to the choice of $\lambda$, even if kept constant throughout training.
Figure~\ref{fig:other_real} (bottom-right) shows welfare (normalized) obtained for a different $\lambda$ on Movielens using $d=12$.
As can be seen, any $\lambda>0.5$ works well and on par with our heuristic choice of $\lambda=1-k/2d$, used in Sec.~\ref{sec:experiments}.


\subsection{No-choice penalty}
\label{appendix: no-choice penalty}
One empirical observation that came up during experimentation was that the optimization of our proposed differential welfare proxy occasionally converged to a degenerate solution in which all users choose the null option. To circumvent this, we added to the objective a penalty term that discourages the outcome in which all users choose the null item (see end of Sec.~\ref{sec:method}). This proved useful in steering the optimization trajectory away from such undesired local optima, which trivially implies no congestion but for the “wrong” reasons, and is by definition sub-optimal (in terms of both the proxy objective value and actual welfare).
A possible concern that could arise from using this penalty would be that it could inadvertently coerce users to always choose some (non-null) item, which is undesirable.
However, Figure \ref{fig:null_choices} (Left) demonstrates that this is not the case: user do indeed choose the null item even when the penalty is present in the objective.
Furthermore, 
and as ban be expected,
such non-choices become even more frequent when prices are higher.

\begin{figure}[t!]
    \centering
    \includegraphics[width=0.38\linewidth]{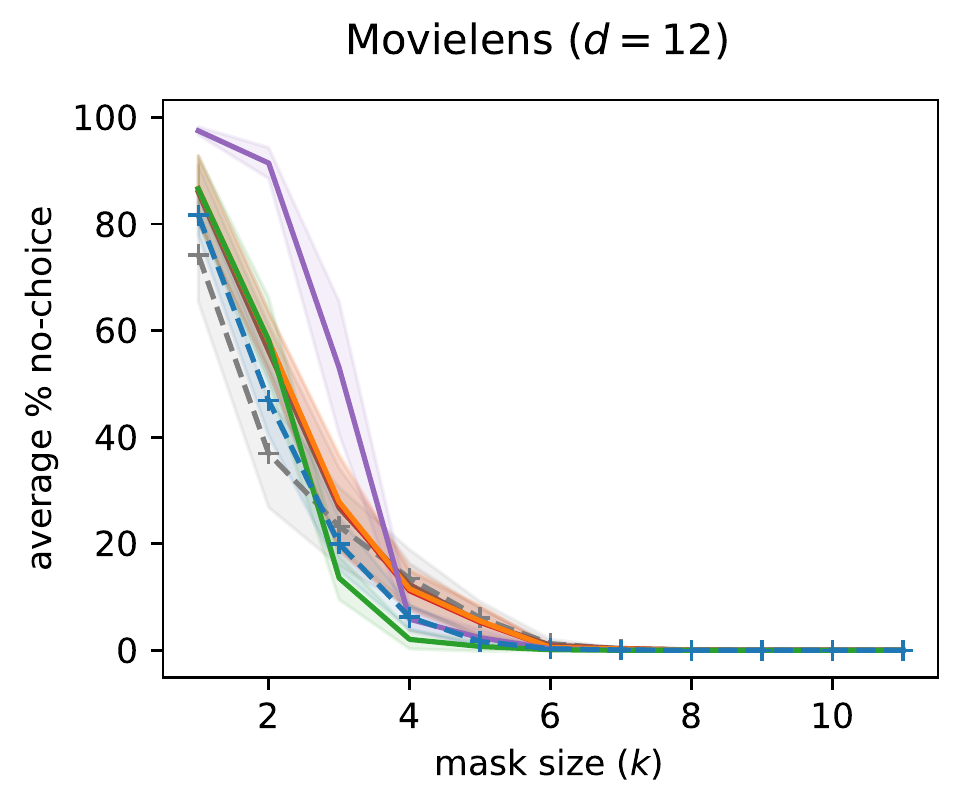} 
    \includegraphics[width=0.5\linewidth]{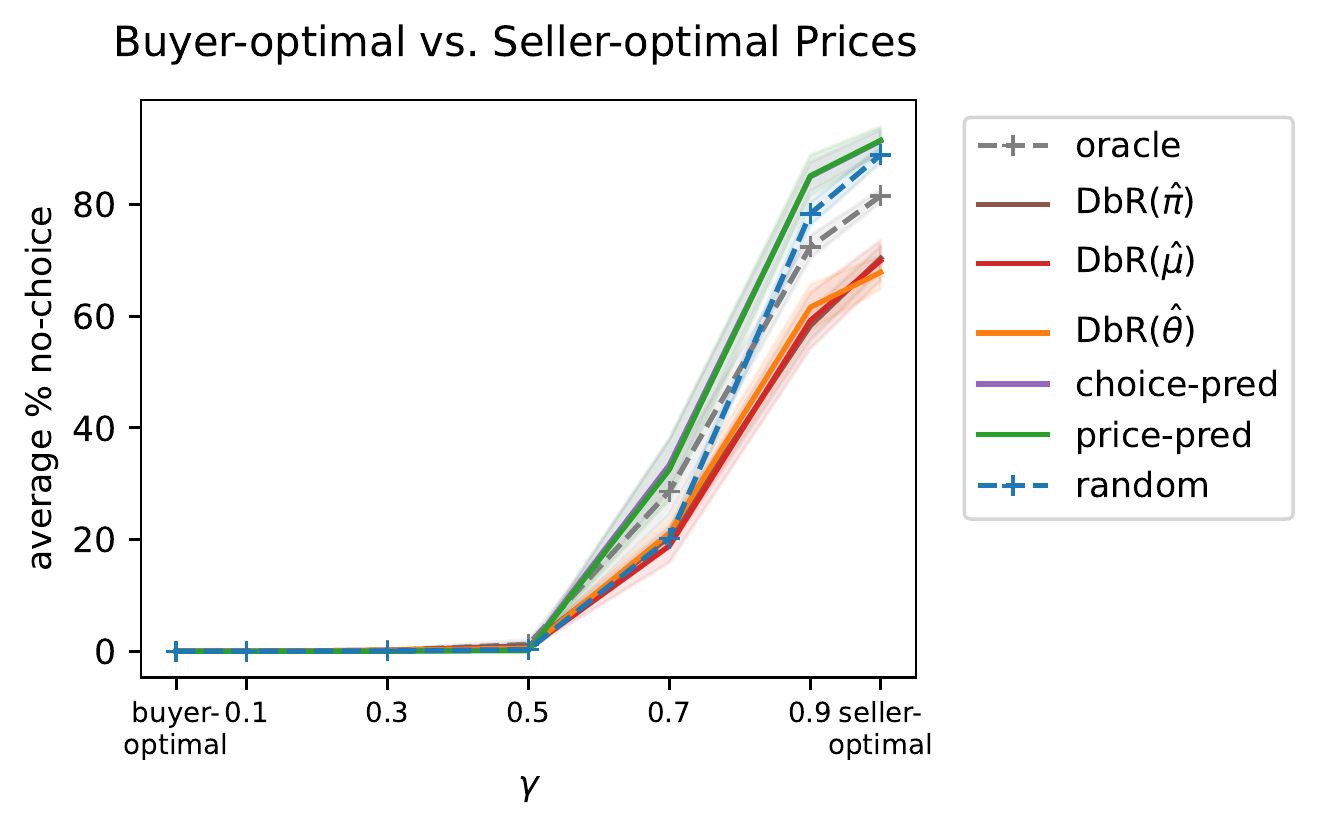}
    \caption{
    Null-item choices (equivalent to no-choice).
    \textbf{(Left:)} Main experiment on $d=12$.
    \textbf{(Right:)} Prices scheme in which CE prices ranges from buyer-optimal (minimal) to seller-optimal (maximal).
    }
    \label{fig:null_choices}
\end{figure}

\subsection{Higher-resolution results for Movielens}
\label{appendix:movielens-extra}
When examining Fig.~\ref{fig:real} from Sec.~\ref{sec:exp_real}, which shows results for the Movielens dataset,
it may seem as if there is a qualitative difference between outcomes for $d=12$ and $d=100$:
whereas for $d=12$ the improvement of the \DbR\ methods in terms of welfare (relative to random) seems to increase with $k$ and then decrease, for $d=100$, it appears to be only increasing.
This, however, is an artifact of the range of values of $k$ considered in each experiment; 
Clearly, for any $d$, performance cannot only increase in $k$, since for $k=d$ performance for all methods is the same, and so the relative gain vs. the random baseline is always zero (as it is also for $k=0$).
Hence, and whereas for $d=12$ performance peaks at around $k=9$, the optimal point for $d=100$ (again, in terms of relative welfare gain) is for some $k \in [90,100]$.

To validate this, we evaluated performance on a tighter grid of values for large $k$,
and in particular for $k \in \{91,92,\dots,99\}$.
Results are shown in Fig.~\ref{fig:movielens-extra}, together with all previous $k$. As expected, for $d=100$ relative welfare gains do indeed increase first and then decrease, with the maximum attained at around $k=96$.

\begin{figure}[t!]
    \centering
    \includegraphics[width=0.55\linewidth]{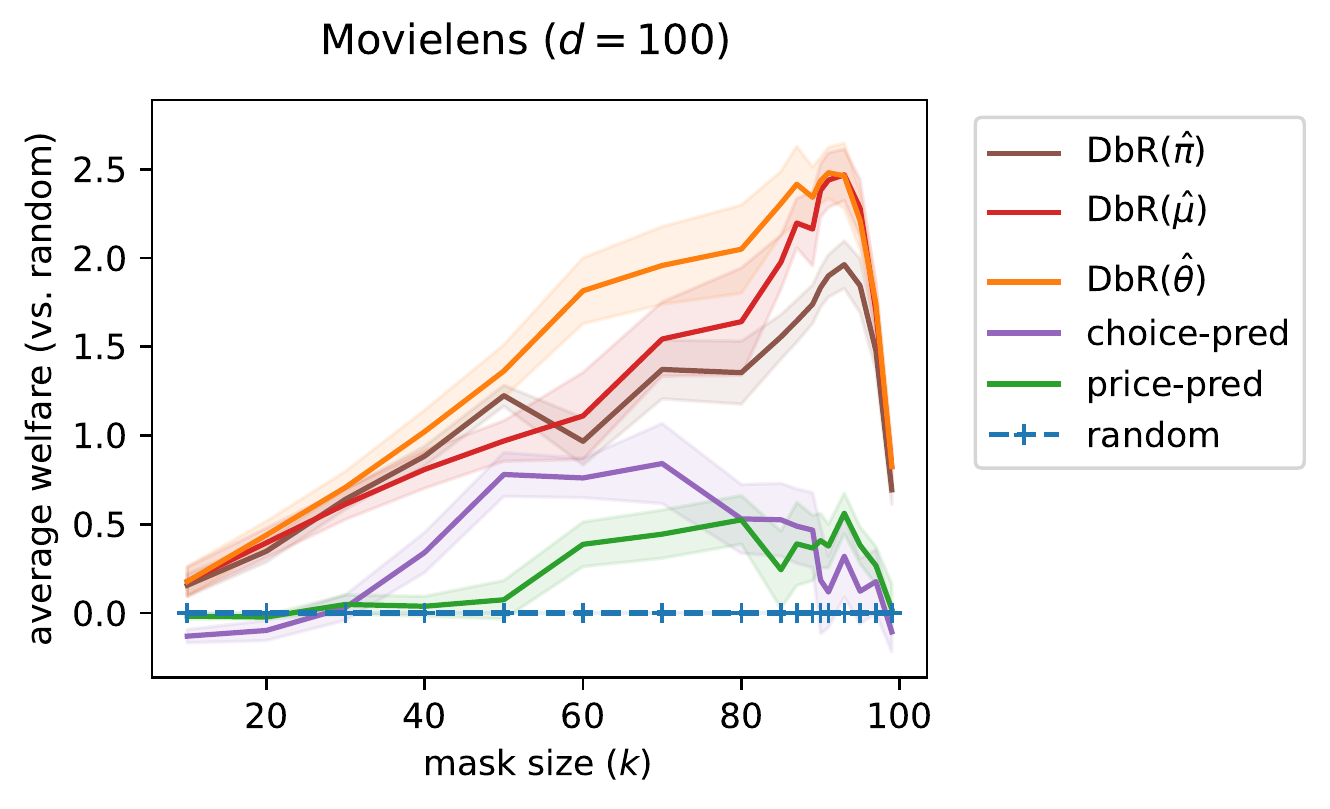} 
    \caption{
    Results on Movielens with $d=100$ for additional values of $k>90$.
    }
    \label{fig:movielens-extra}
\end{figure}

\subsection{Relative and absolute performance}
In Sec.~\ref{sec:experiments}, for our experiments which vary $k$,
we chose to portray results normalized from below to match random performance (\random). This was mainly since the overall effect on performance of increasing $k$ is larger than that which can be obtained by any method (i.e., the gap between \random\ and \oracle).
For completeness, Figure~\ref{fig:other_real} (top row)
shows unormalized results, which show in absolute terms how overall performance increases for $k$.
Figure~\ref{fig:other_real} (bottom-left) shows results normalized from both below (matching \random) and above (matching \oracle);
as can be seen, our approach obtains fairly constant relative performance across $k$.
For completeness, Figure~\ref{fig:num_unique} shows in more detail the number of allocated items for the $d=12$ setting.

\begin{figure}[b!]
    \centering
    \includegraphics[width=0.43\linewidth]{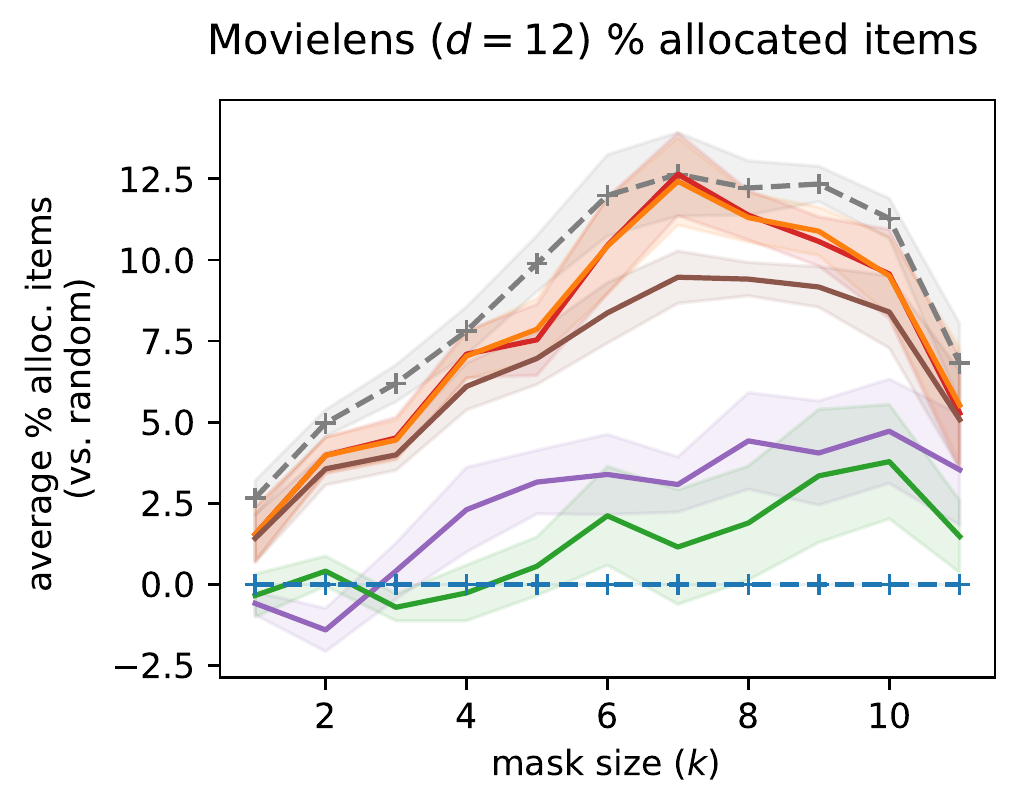} 
    \caption{
    Number of unique items (enlarged version of inlay in Fig.~\ref{fig:real} (left)).
    }
    \label{fig:num_unique}
\end{figure}

\begin{figure}[p!]
    \centering
    \qquad
    \includegraphics[width=0.37\linewidth]{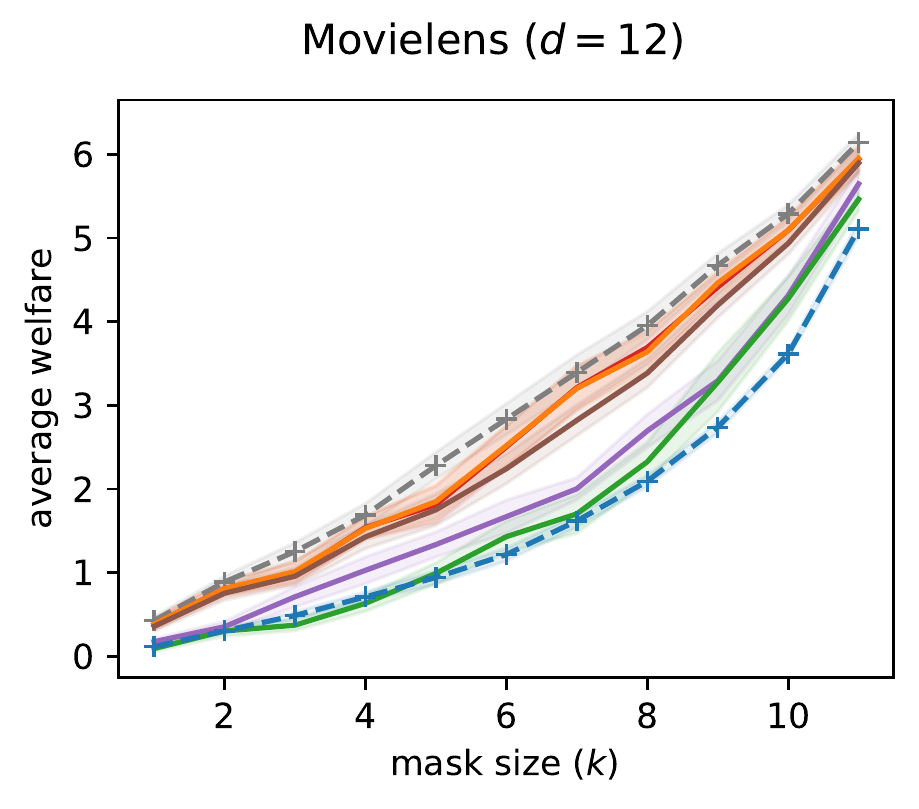} 
    \includegraphics[width=0.37\linewidth]{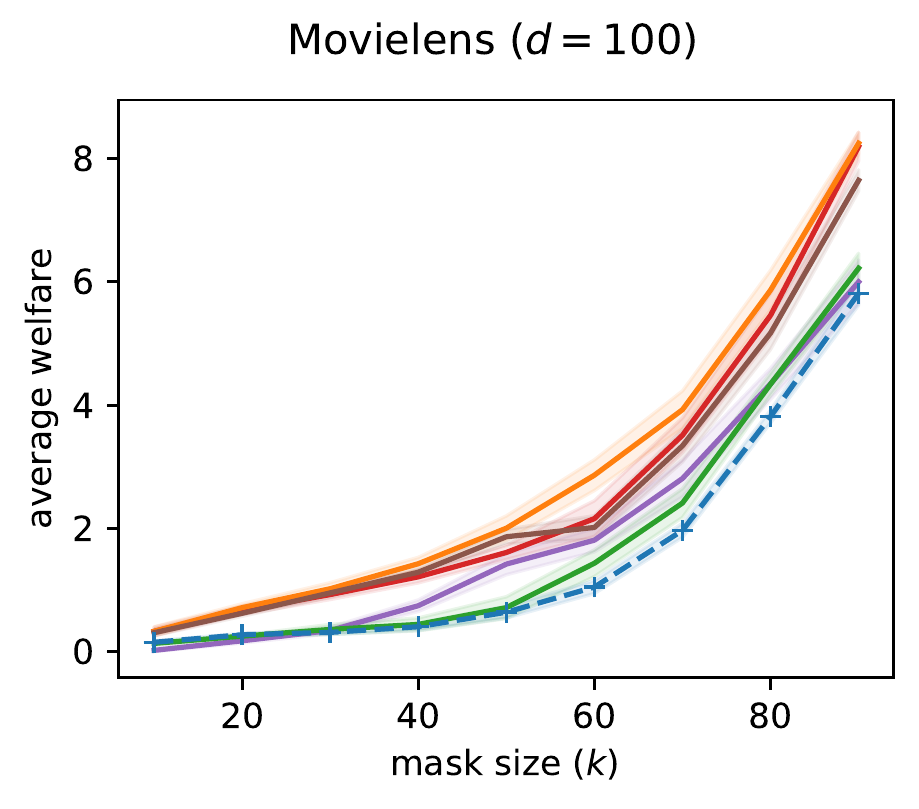}     
    \includegraphics[width=0.18\linewidth]{graphics/real_data/legend.pdf}
    \\
    \includegraphics[width=0.41\linewidth]{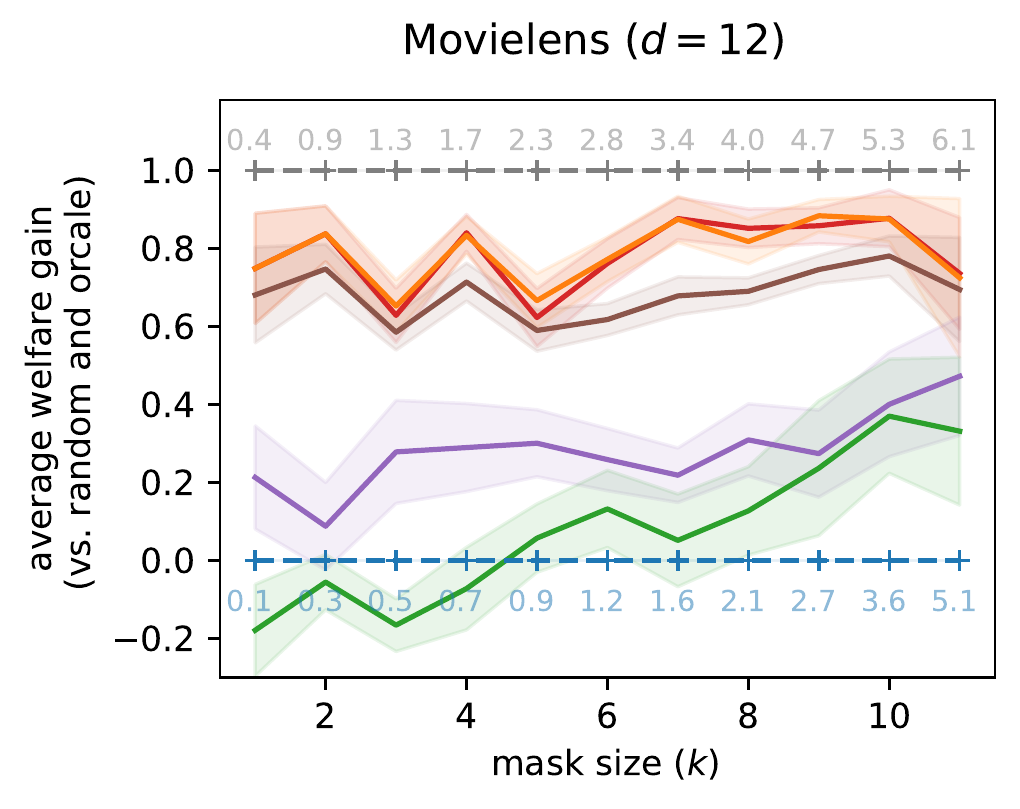} \,\,
    \includegraphics[width=0.51\linewidth]{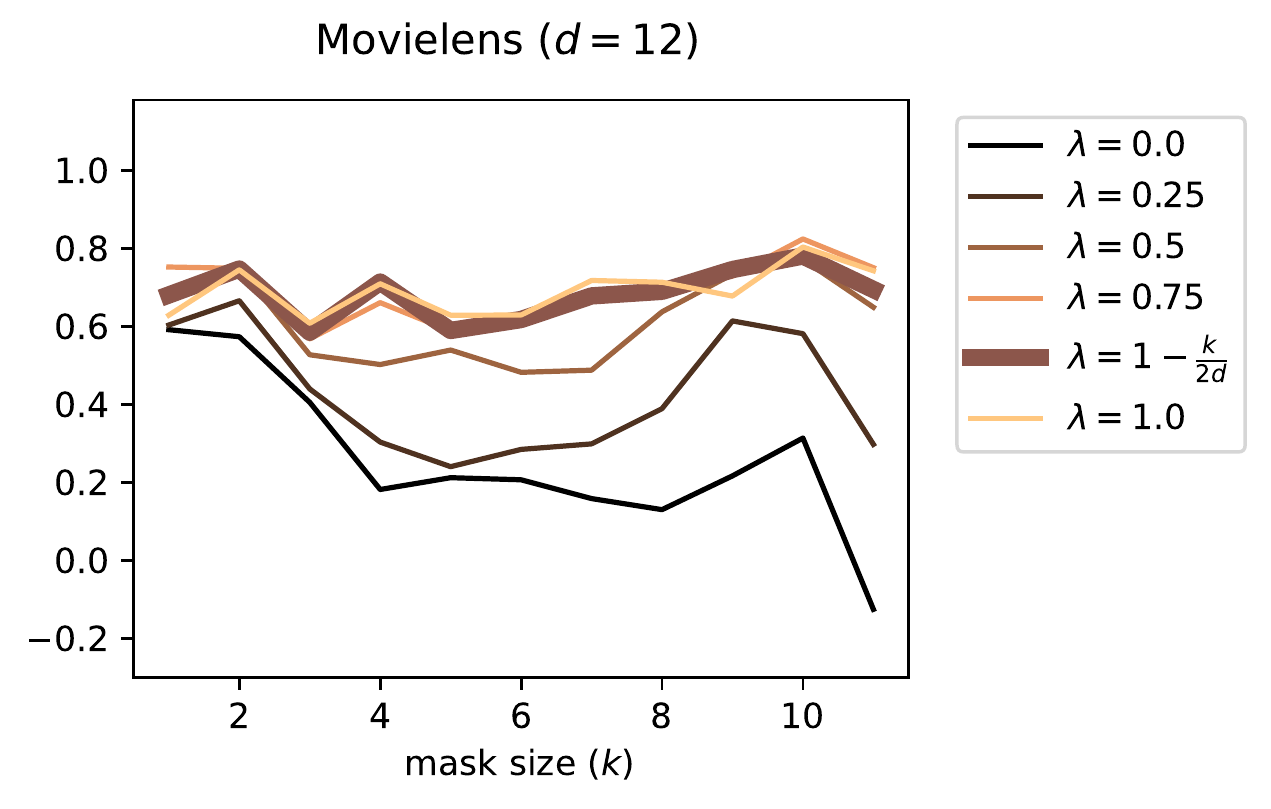} 
    \caption{
    \textbf{(Top row:)}
    Unormalized results, showing how potential and obtained welfare increase with $k$.
    \textbf{(Bottom-left:)}
    Results normalized to 1 from above (matching the \oracle, as in the main paper) and below to 0 (matching \random).
    Our approach shows relative performance that is fairly constant across $k$.
    \textbf{(Bottom-right:)}
    Performance (also normalized) for various fixed $\lambda$.
    }
    \vspace{128in}
    \label{fig:other_real}
\end{figure}


\end{document}